\documentclass[a4paper,fleqn]{cas-sc}
\usepackage[numbers]{natbib}

\usepackage{amsmath,amsfonts,amssymb}
\usepackage{amsthm}
\usepackage{float}
\usepackage{tikz}
\usetikzlibrary{positioning}

\newtheorem{corollary}{Corollary}
\newtheorem{proposition}{Proposition}
\newtheorem{example}{Example}

\begin{document}

\let\WriteBookmarks\relax
\def\floatpagepagefraction{1}
\def\textpagefraction{.001}
\shorttitle{Kernel matrix spectrum estimation}
\shortauthors{M.~Lepilov}

\author[1]{Mikhail Lepilov}[type=editor]
\affiliation[1]{organization={Department of Mathematics, Emory University},
                city={Atlanta},
                state={Georgia},
                country={USA}}
\ead{mlepilov@gmail.com}

\fntext[fn1]{The research of Mikhail Lepilov was supported in part by NSF grant DMS-2038118.}

\title[mode=title]{Fast spectrum estimation of some kernel matrices}


\begin{abstract}
In data science, individual observations are often assumed to come independently from an underlying probability space. Kernel matrices formed from large sets of such observations arise frequently, for example during classification tasks. It is desirable to know the eigenvalue decay properties of these matrices without explicitly forming them, such as when determining if a low-rank approximation is feasible. In this work, we introduce a new eigenvalue quantile estimation framework for some kernel matrices. This framework gives meaningful bounds for all the eigenvalues of a kernel matrix while avoiding the cost of constructing the full matrix. The kernel matrices under consideration come from a kernel with quick decay away from the diagonal applied to uniformly-distributed sets of points in Euclidean space of any dimension. We prove the efficacy of this framework given certain bounds on the kernel function, and we provide empirical evidence for its accuracy. In the process, we also prove a general interlacing-type theorem for finite sets of numbers. Additionally, we indicate an application of this framework to the study of the intrinsic dimension of data, as well as several other directions in which to generalize this work.
\end{abstract}


\maketitle

\section{Introduction}
\label{sec:introduction}

\subsection*{Background} Kernel matrices that result from applying a positive-definite function pairwise to a finite set of points $X\subseteq\mathbb{R}^D$ arise in several areas of computational mathematics such as image processing and machine learning. In the latter field especially, common methods involve performing expensive computations with a kernel matrix, such as inverting it or finding its eigenvalues \cite{scholkopf,shawetaylor}. The kernel matrix involved, however, may be of a prohibitively large size to form, let alone to perform computations with. On the other hand, if the matrix has quick eigenvalue decay relative to its norm, then we may be able to efficiently carry out computations on its low-rank approximation instead. A good overview of such computations and their complexity is found in \cite{cesabianchi}. Hence, it is useful to study {\it a priori} the eigenvalue decay of a kernel matrix. Given the $n$ data points with which the kernel matrix is formed, we would like to find ways to estimate all of its eigenvalues faster than by having to form the matrix first. That is, we would like to do so in a number of operations that is at most subquadratic, and ideally sublinear, relative to $n$.

We consider a setting common in data science, which is when the points in $X$ are assumed to be independent and identically-distributed, coming from some latent distribution. In the past, the study of eigenvalue decay of such kernel matrices often focused on asymptotic eigenvalue behavior as the number of distribution samples in $X$ was taken to infinity, after making some appropriate assumptions on the distribution and kernel function involved \cite{kolgine,braun}. However, as the examples in \cite{braun} suggest, these bounds rely on the kernel function having its truncated eigendecomposition (in some appropriate function space) readily available. Furthermore, it is unclear exactly how many terms to keep when computing and truncating such an eigendecomposition in order to obtain an eigenvalue decay bound within some tolerance. Thus, it is impractical to use such ideas for our purposes of estimating eigenvalue decay of a given kernel matrix.

These difficulties are sidestepped when empirical methods, such as matrix sketching, are used to obtain bounds on eigenvalues. However, most sketching techniques typically require not only forming the kernel matrix but also finding matrix-vector products with sets of specially-crafted vectors. For some examples and an overview, see \cite{woodruff,swatwood}. Such techniques applied to an $n\times n$ matrix, therefore, would require a number of operations that scales at least quadratically in $n$, so most sketching approaches do not provide a way to achieve our goal. One exception is the class of techniques known as Nystr\"{o}m methods, which can be thought of as matrix sketching methods that do not require forming the entire kernel matrix. In Nystr\"{o}m methods, a random subsample of the points in $X$, and hence of the kernel matrix, is used to obtain a low-rank decomposition of the full matrix. The spectrum of this randomly-subsampled matrix is shown to be correlated pointwise with the first few eigenvalues of the full matrix \cite{willseeger}. Various strategies for sampling the matrix and obtaining theoretical pointwise accuracy guarantees for this correlation have been implemented over the years. Such guarantees depend on performing additional computations with the data points informing the choice of samples; see, for example, \cite{drinmahon}. An in-depth empirical exploration of such guarantees, and especially their limitations, is given in \cite{gitmahon}. However, since the goal of such methods is to find the best low-rank approximation, and not to find whether or not a good such approximation exists, these accuracy guarantees only apply to give eigenvalue estimates for the first few eigenvalues. Furthermore, in practice, the low-cost ``naive" Nystr\"{o}m method of \cite{willseeger} actually does not work to give a subsampled matrix with similar eigenvalues if the matrix has high numerical rank; see Figure~\ref{fig:nystrom} for an illustration of this phenomenon.

\begin{figure}[ptbh]
\centering
\begin{minipage}{0.5\textwidth}
\begin{tikzpicture}
  \node (img)  {\includegraphics[scale=0.22]{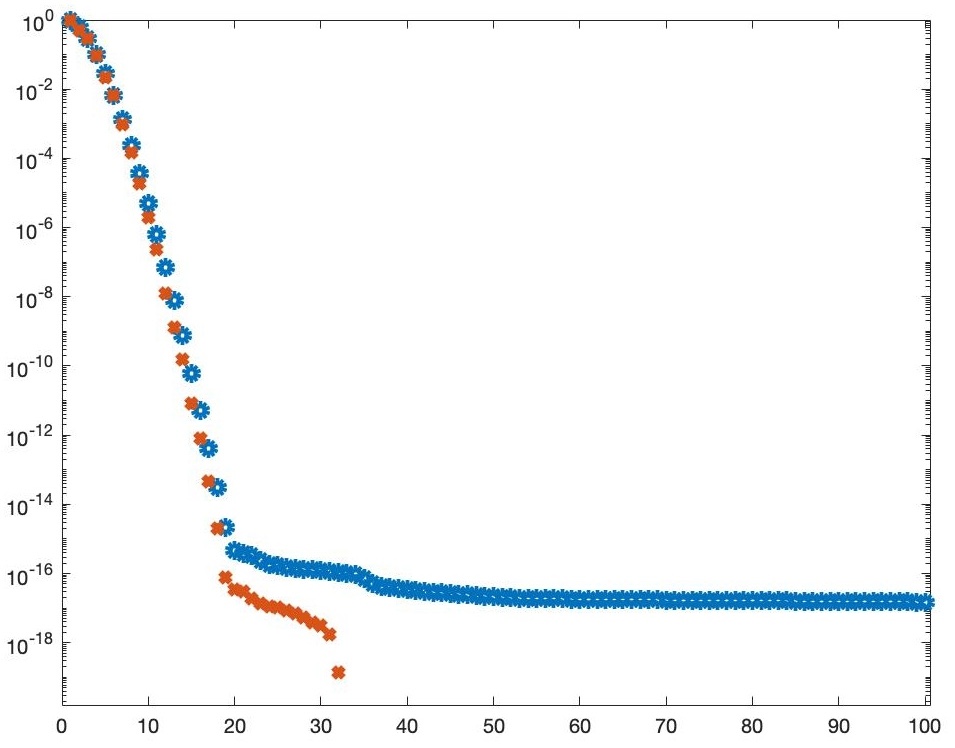}};
  \node[left=of img, node distance=0cm, rotate=90, anchor=center,yshift=-0.7cm] {Magnitude};
 \end{tikzpicture}
\end{minipage}\\

\begin{minipage}{0.5\textwidth}
\begin{tikzpicture}
  \node (img)  {\includegraphics[scale=0.22]{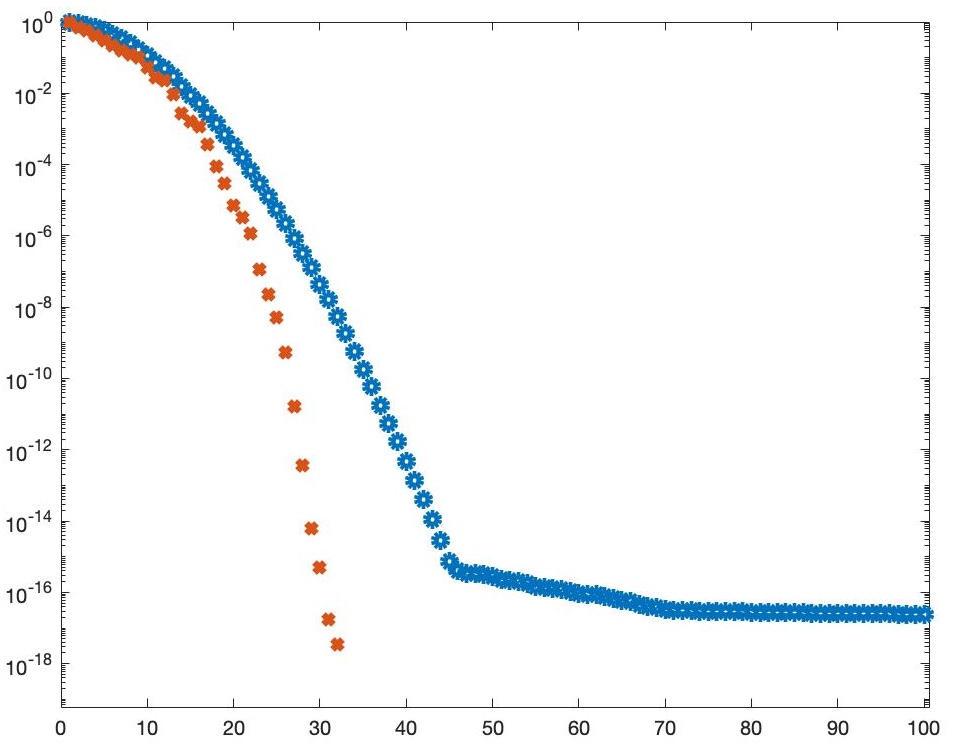}};
  \node[left=of img, node distance=0cm, rotate=90, anchor=center,yshift=-0.7cm] {Magnitude};
 \end{tikzpicture}
\end{minipage}\\

\begin{minipage}{0.5\textwidth}
\;\begin{tikzpicture}
  \node (img) {\includegraphics[scale=0.22]{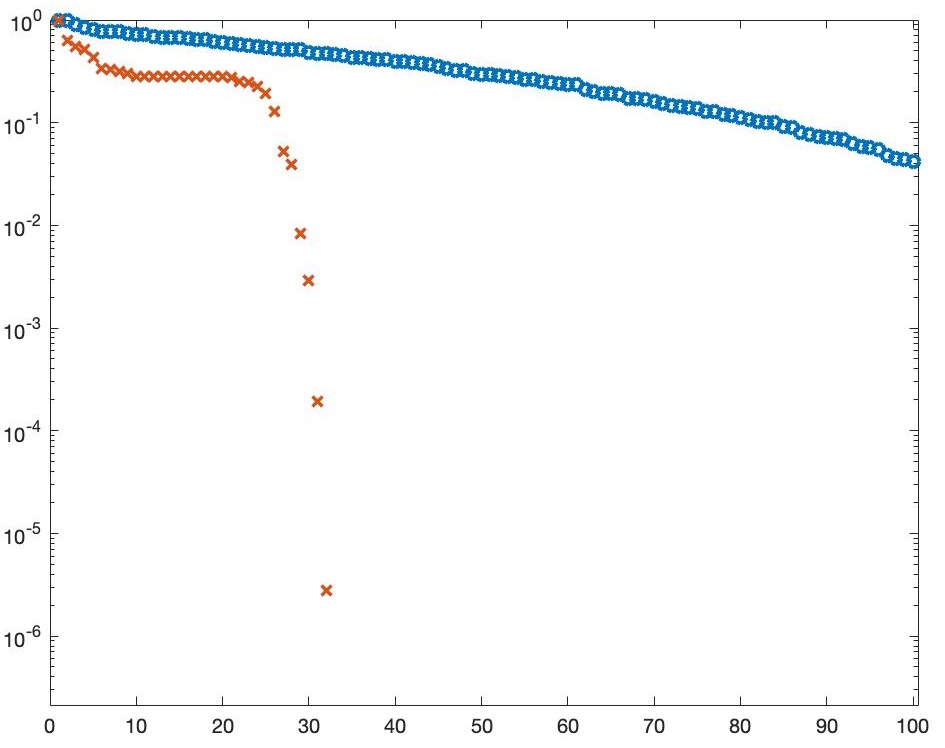}};
  \node[below=of img, node distance=0cm, yshift=1cm] {Index of eigenvalue};
  \node[left=of img, node distance=0cm, rotate=90, anchor=center,yshift=-0.7cm] {Magnitude};
 \end{tikzpicture}
\end{minipage}
\caption{The first 100 eigenvalues of the kernel matrix (blue circles) formed when $X$ consists of 512 points taken from the standard uniform distribution in one dimension, as well as those of its ``naive" Nystr\"{o}m approximation (red crosses) with 32 points. Here, the kernel used is $\kappa(x,y)=e^{-10(x-y)^2}$ (top figure), $\kappa(x,y)=e^{-100(x-y)^2}$ (middle figure), and $\kappa(x,y)=e^{-10000(x-y)^2}$ (bottom figure). It is evident that, in the top figure, the eigenvalue decay of the subsampled matrix corresponds well with the eigenvalue decay of the full matrix, but in the center and especially bottom figures, this is no longer the case. This indicates that the Nystr\"{o}m method only works to give an estimate of numerical rank if we know {\it a priori} that it is low for our given kernel matrix, as in the top figure.}
\label{fig:nystrom}
\end{figure}

Even more recently, related work comes from approximating graph spectra in subquadratic time, such as in \cite{bakshi,cohen-steiner}. In this approach, the kernel matrix can be regarded as the Laplacian of a particular weighted complete graph. Specifically, each vertex corresponds to one point, and each edge has weight equal to the kernel evaluated at the points corresponding to the vertices that the edge connects. Methods based on this are different from Nystr\"{o}m methods and instead give global distributional bounds in the Wasserstein-1 metric, often referred to as the ``earth-mover distance." From this, however, it is difficult to obtain pointwise estimates of the matrix spectrum. The reference \cite{bakshi} does contain such estimates for the first few eigenvalues but not for the later eigenvalues.

Perhaps the closest approach to the one we undertake may be found in \cite{bhattacharjee}. This is a result for general symmetric matrices that, in its basic form, gives additive bounds unrelated to the magnitude of each eigenvalue for the later eigenvalues. This makes controlling errors difficult for the later eigenvalues, and it prevents us from using the approach if the numerical rank of the matrix is not already low.

Hence, to obtain accurate pointwise estimates for all the eigenvalues of a given kernel matrix in subquadratic time, we must find a new empirical approach that avoids the issues of the methods above. To do so, we first note that all of the methods we mention so far use no more information than just the fact that the matrix is symmetric. Thus, using more information about the distribution underlying $X$, as well as the kernel involved in forming the matrix, may enable us to find a better approximation for its spectrum.

\subsection*{Our contribution}
In this work, we fix the kernel to a narrow class and the distribution underlying $X$ to the uniform distribution to design a proof-of-concept for a fundamentally new estimation technique for the eigenvalues of the resulting kernel matrix. The technique is based on finding bounds for the expected $k$ quantiles of the eigenvalues, for the case that $k\ll n$. This is done, in turn, by matching the moments of this eigenvalue distribution with that of a smaller, $k\times k$ matrix formed randomly from values of the original matrix. The matrix subsampling approach is conceptually similar to that of graph-theoretic spectrum estimation techniques of \cite{bhattacharjee}, and the moment-based distributional bounds are related to those used in Stochastic Lanczos Quadrature \cite{chen1}. Empirical evidence suggests that this technique works when the kernel in question has quick decay away from the diagonal, which corresponds to the case that the matrix is of high numerical rank. This complements the existing methods mentioned above, which do not give good accuracy guarantees in such cases (again, see Figure~\ref{fig:nystrom}).

Using an unoptimized eigenvalue algorithm on the smaller $k\times k$ matrix, this new framework requires at most $O(mk^3)$ computations, where $m$ is a constant that depends on the desired approximation accuracy. Although it is true the matrices under consideration may be approximated by banded or sparse matrices, this would require at least $n$ diagonal matrix values. Thus, for certain kernels used to compute a kernel matrix $A$, our new framework allows for the only method to our knowledge to find bounds on the later eigenvalues of the resulting kernel matrix that does not scale with $n$, after a preprocessing step that does not depend on the matrix or kernel. In addition, since this is an entirely new approach, it provides a natural set of questions for further study that could allow subquadratic eigenvalue estimates for wider classes of matrices. Along the way, we also show a general result concerning the interlacing of sets of real numbers, which is conceptually related to some results on Gaussian quadrature rules that have recently been studied in the context of spectrum density estimation. Finally, we propose an application of this work to the problem of finding the so-called intrinsic dimension of a dataset.

The rest of the paper is structured as follows: in Section~\ref{sec:background}, we detail our approach. In the process, we prove several new results that show its efficacy in kernel matrix eigenvalue quantile estimation. Among these results is the aforementioned new, general interlacing result concerning finite sets of real numbers. In Section~\ref{sec:numerical}, we give some numerical experiments showing the strengths and limitations of our new framework. Finally, in Section~\ref{sec:futurework}, we pose a number of questions for further study that could improve the framework. We also suggest an application to the problem of dimension reduction in data science.

Throughout the paper, we use the following notation. Let $D,n\in\mathbb{N}$, and let $X\subseteq\mathbb{R}^D$ with $|X|=n$. Let $\kappa:\mathbb{R}^D\times\mathbb{R}^D\to\mathbb{R}$ be a symmetric, positive-definite function. Fix an indexing $X=\{\mathbf{x_1},\ldots,\mathbf{x_n}\}$. By $\kappa(X,X)$, we mean the kernel matrix $A\in\mathbb{R}^{n\times n}$ with entries $A_{ij}=\kappa(\mathbf{x}_i,\mathbf{x}_j)$. For a symmetrix matrix $A\in\mathbb{R}^{n\times n}$ and some $1\leq j\leq n$, we denote by $\sigma_j(A)$ the $j$th largest eigenvalue of $A$. Finally, for $a,b\in\mathbb{R}$ with $a\leq b$, we denote by $U[a,b]^D$ the uniform distribution on the cube $[a,b]^D$.

\section{Theoretical results}
\label{sec:background}
Fix $X$ and $A$ as above. We will assume throughout the paper that each $\mathbf{x}_i\sim U[0,1]^D$, but we will comment later on how we may relax this assumption to obtain more general analogs of our main ideas. We concern ourselves with finding bounds for the eigenvalues of $A$.

We do so by finding another kernel matrix $B\in\mathbb{R}^{k\times k}$, for $k\ll n$, formed using $k$ points sampled from among the $\mathbf{x}_i$s. We wish for the $k$ eigenvalues of $B$ to then give bounds for the $k$ quantiles of the eigenvalue distribution of $A$ in the following way. Without loss of generality, we may assume $k|n$. We wish for $B$ to have the property that
\begin{equation}
\label{eq:interlacing}
\sigma_{\lceil\frac{jk}{n}\rceil-1}(B)\geq\sigma_j(A)\geq\sigma_{\lceil\frac{jk}{n}\rceil+1}(B),
\end{equation}
for $1\leq j\leq n$, where we define ``$\sigma_0(B)=\infty$" and ``$\sigma_{k+1}(B)=0$." In other words, we wish for each $n/k$ consecutive eigenvalues, ordered  of $A$ to be ``sandwiched" between two of the $k$ eigenvalues of $B$, which we may compute in at worst $O(k^3)$. We may look ahead to Figure~\ref{fig:momentinterlace} for a picture of this, but we first state our motivation. The reason we wish to find another matrix $B$ using a subsample of the original points, heuristically, is to preserve information about the geometry of the distribution that gives rise to the $\mathbf{x}_i$s. An implicit assumption is that $n$ is so large compared to $k$ that picking $k$ of the $\mathbf{x}_i$s (i.e. sampling from entries of $A$) is the same as drawing from the original distribution.

\subsection{Interlacing property of sets of real numbers} We may expect to get something like the bounds in \eqref{eq:interlacing} if we match each of the $k$ moments of the empirical spectral distributions of $A$ and $B$, which are defined as the discrete uniform distributions on $\mathcal{A}=\{\sigma_1(A),\ldots,\sigma_n(A)\}$ and $\mathcal{B}=\{\sigma_1(B),\ldots,\sigma_k(B)\}$, respectively. This is because of the usual notion that the moments of a distribution convey its shape. In the case of the discrete uniform distribution on $\mathcal{B}$, we know that all shape information is contained entirely in its first $k$ moments, since $\mathcal{B}$ contains only $k$ points. Hence, we may informally think of matching each of the $k$ moments of $\mathcal{A}$ and $\mathcal{B}$ as the best we can do in terms of estimating quantiles.

A moment-matching framework to gain information about the empirical spectral density of a symmetric positive definite (SPD) matrix was recently outlined in the context of the Lanczos algorithm and related methods in the review of Lin, Saad, and Yang \cite{lin}. This approach is generally referred to as Stochastic Lanczos Quadrature (SLQ). In brief, the approach consists of performing a specified number of iterations of Lanczos on a random starting vector, and then using the resulting Ritz values with certain weights to obtain spectrum information. The analysis and properties of SLQ, as well as those of the related kernel polynomial method, was extensively studied by Chen, Trogdon, and Ubaru in \cite{chen1} and \cite{chen2}, in which resulting global error bounds on the estimated spectral density are obtained.

Our approach also relies on matching moments, but it distinct from the above in that it provides ``unweighted" spectral density bounds, which allows us to extract quantile information. Furthermore, our approach relies on graph-theoretic matrix subsampling methods in the spirit of \cite{bhattacharjee}, rather than matrix multiplication method via the Lanczos algorithm, to gain information about the spectrum. This is made possible by our additional assumptions on how $A$ is formed, namely by using a fast-decaying kernel applied to uniformly-distributed random variables.

We start by stating proposition below, which is a general property of sets of real numbers. Note that, for convenience of notation, we assume henceforth that all the eigenvalues of $A$ are distinct and all the eigenvalues of $B$ are distinct. In practice, this assumption holds if the underlying distribution of $X$ is continuous and the kernel is strictly decreasing away from the diagonal. However, the following proposition and corollary can be easily modified to hold even in the case of repeated eigenvalues.

\begin{proposition}
\label{prop:momentmatching}
Let $S,T\subseteq\mathbb{R}_{\geq0}$ with $|S|=n$, $|T|=k$, and $k|n$. Denote by $s_i$ and $t_j$ the $i$th and $j$th largest elements of $S$ and $T$, respectively, and suppose $\sum_{i=1}^n\frac{s_i^r}{n}=\sum_{i=1}^k\frac{t_i^r}{k}$ for all $r=1,\ldots,k$. Then
\begin{equation*}
t_{\lceil\frac{jk}{n}\rceil-1}\leq s_j\leq t_{\lceil\frac{jk}{n}\rceil+1}
\end{equation*}
for all $j=1,\ldots,n$, where we define $t_0=0$ and $t_{k+1}=\infty$. (See Figure~\ref{fig:momentinterlace} for an illustration of this.)
\end{proposition}

\begin{figure}[ptbh]
\centering
\begin{tikzpicture}
  \node (img)  {\includegraphics[width=0.7\textwidth]{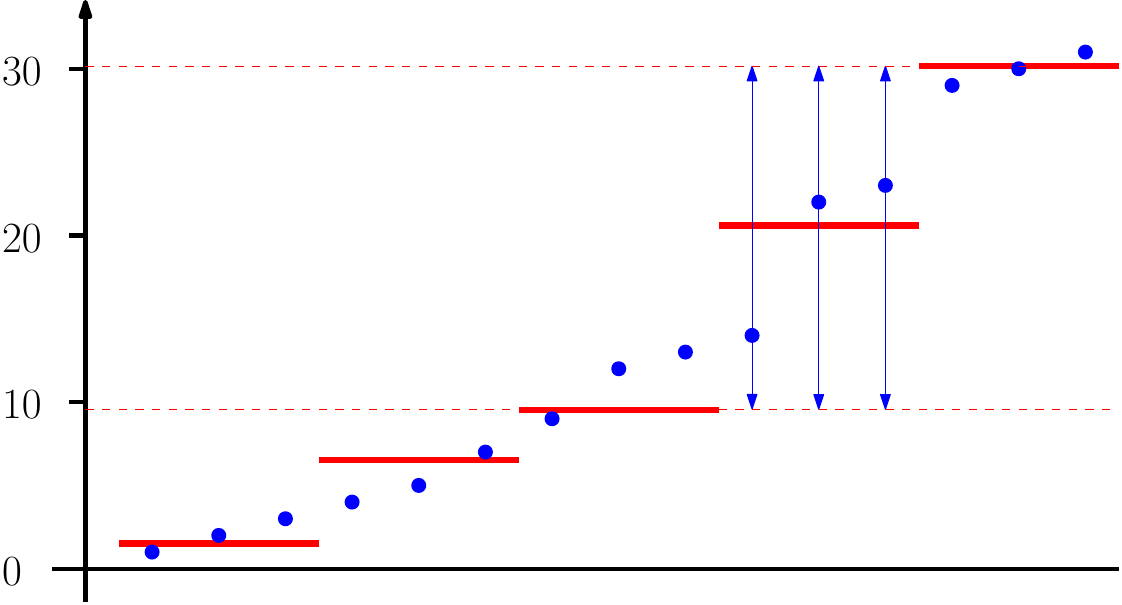}};
  \node[below=of img, node distance=0cm, yshift=1cm] {Index of element of $S$, $T$ (sorted by magnitude)};
  \node[left=of img, node distance=0cm, rotate=90, anchor=center,yshift=-0.7cm] {Magnitude};
 \end{tikzpicture}
\caption{The sets $S=\{1,2,3,4,5,7,9,12,13,14,22,23,29,30,31\}$ (blue dots) and $T$ (solid red dashes), where $T$ is picked such that $\sum_{i=1}^{15}s_i^r/15=\sum_{i=1}^5t_i^r/5$ for $r=1,\ldots,5$. Hence, $T$ is approximately $\{1.51216,6.52312,9.54601,20.5897,30.1624\}$. Proposition~\ref{prop:momentmatching} shows, for example, that $t_3\leq s_{10},s_{11},s_{12}\leq t_5$. This is illustrated with the blue arrows above.}
\label{fig:momentinterlace}
\end{figure}

\begin{proof}
Consider the discrete uniform probability distributions on $S$ and $T$, with the former having cumulative distribution function $F_S$. Then denoting by $\mu_i$ and $\nu_i$ the $i$th moments of these distributions on $S$ and $T$ for $i=0,\ldots,k$, respectively, our assumptions are equivalent to requiring that $\mu_i=\nu_i$ for each $i=1,\ldots,k$. The statement follows as a quick corollary to some classical results on the bounds for $F_S$ in terms of its moments, which we reproduce here.

Following the notation and presentation of \cite{momentbook}---in particular, note the relationships in Equations~1.3 and 1.4 of Chapter 1---we construct the set of polynomials $P_0,\ldots,P_k$ by the explicit formulas $P_0=1$ and
\begin{equation*}
P_j(x)=\frac{1}{\sqrt{D_{j-1}D_j}}\begin{vmatrix}\mu_0 & \mu_1 & \cdots & \mu_j\\
\mu_1 & \mu_2 & \cdots & \mu_{j+1}\\
\vdots & \vdots & \ddots & \vdots\\
\mu_{j-1} & \mu_j & \cdots & \mu_{2j-1}\\
1 & x & \cdots & x^j\end{vmatrix}
\end{equation*}
for $j=1,\ldots,k$, where
\begin{equation*}
D_j=\begin{vmatrix}\mu_0 & \mu_1 & \cdots & \mu_j\\
\mu_1 & \mu_2 & \cdots & \mu_{j+1}\\
\vdots & \vdots & \ddots & \vdots\\
\mu_j & \mu_{j+1} & \cdots & \mu_{2j}\end{vmatrix}
\end{equation*}
for $j=0,\ldots,k$. These polynomials satisfy a number of properties, but here we note only the following: if we write the product $P_iP_j$ as $(P_iP_j)(x)=\sum_{l=0}^{\deg(P_i)\deg(P_j)}c_{i,j,l}x^l$ for some coefficients $c_{i,j,l}$, then
\begin{equation*}
\sum_{l=0}^{\deg(P_i)\deg(P_j)}c_{i,j,l}\mu_l=\delta_{i,j}.
\end{equation*}
Since $\nu_i=\mu_i$ for all $i$ and because $T$ has the discrete uniform distribution, this is equivalent to
\begin{align}
\sum_{l=1}^k\frac{(P_iP_j)({t_l})}{k}&=\delta_{i,j},\mathrm{\;or}\nonumber\\
\sum_{l=1}^k(P_iP_j)(t_l)&=k\delta_{i,j}\label{eq:kroneckerprop},
\end{align}
where $\delta_{i,j}$ is the Kronecker delta. (In other words, in our case, $(P_i)_{i=0,\ldots,k}$ is a sequence of polynomials orthogonal with respect to the average of the evaluation functionals at the $t_j$s for $j=1,\ldots,k$.) Furthermore, following \cite{momentpaper}, we construct the ``empirical Christoffel function"
\begin{equation*}
\lambda_k=\frac{1}{\sum_{i=0}^kP_i^2}
\end{equation*}
Now, let $x_i$ for $i=1,\ldots,k$ be the roots of $P_k$. Using the function $\lambda$, the authors in \cite{momentpaper} note the following bounds on $F_S$:
\begin{equation*}
1-\sum_{j=i}^k\lambda(x_j)\leq F_S(x_i)\leq\sum_{j=1}^i\lambda(x_j).
\end{equation*}
By our definition of $F_S$, the proposition therefore follows if we show that (i) the $t_i$'s are precisely the roots $x_i$ of $P_k$, and (ii) $\lambda(x_i)=1/k$ for each $i=1,\ldots,k$. To see (i), we note that since $\mu_i=\nu_i$ for all $i\in\mathbb{N}$, for each $t_i$ we have
\begin{align*}
P_k(t_i)&=\frac{1}{\sqrt{D_{k-1}D_k}}\begin{vmatrix}\left(\frac{1}{k}\right)\sum_{l=1}^k1 & \left(\frac{1}{k}\right)\sum_{l=1}^kt_l & \cdots & \left(\frac{1}{k}\right)\sum_{l=1}^kt_l^k\\
\left(\frac{1}{k}\right)\sum_{l=1}^kt_l & \left(\frac{1}{k}\right)\sum_{l=1}^kt_l^2 & \cdots & \left(\frac{1}{k}\right)\sum_{l=1}^kt_l^{k+1}\\
\vdots & \vdots & \ddots & \vdots\\
\left(\frac{1}{k}\right)\sum_{l=1}^kt_l^{k-1} & \left(\frac{1}{k}\right)\sum_{l=1}^kt_l^k & \cdots & \left(\frac{1}{k}\right)\sum_{l=1}^kt_l^{2k-1}\\
1 & t_i & \cdots & t_i^k\end{vmatrix}\\
&\scalebox{0.85}{$=\frac{1}{\sqrt{D_{k-1}D_k}}
\begin{vmatrix}\begin{bmatrix}\left(\frac{1}{k}\right) & \left(\frac{1}{k}\right) & \cdots & \left(\frac{1}{k}\right) & \cdots & \left(\frac{1}{k}\right)\\
\left(\frac{1}{k}\right)t_1 & \left(\frac{1}{k}\right)t_2 & \cdots & \left(\frac{1}{k}\right)t_i & \cdots & \left(\frac{1}{k}\right)t_k\\
\vdots & \vdots & \vdots & \vdots & \ddots & \vdots\\
\left(\frac{1}{k}\right)t_1^{k-1} & \left(\frac{1}{k}\right)t_2^{k-1} & \cdots & \left(\frac{1}{k}\right)t_i^{k-1} & \cdots & \left(\frac{1}{k}\right)t_k^{k-1}\\
0 & 0 & \cdots & 1 & \cdots & 0\end{bmatrix}\begin{bmatrix}1 & t_1 & t_1^2 & \cdots & t_1^k\\
1 & t_2 & t_2^2 & \cdots & t_2^k\\
\vdots & \vdots & \cdots & \ddots & \vdots\\
1 & t_k & t_k^2 & \cdots & t_k^k\end{bmatrix}\end{vmatrix}$}\\
&=0.
\end{align*}

Now, note that by fact (i), we see that (ii) is equivalent to the condition that $\sum_{i=0}^kP_i^2(t_i)=k$ for each $i=1,\ldots,k$. Define the matrix $C$ by
\begin{equation*}
C_{j,m}=\sum_{i=0}^{k-1}
\scalebox{0.85}{$\tiny
\setlength{\arraycolsep}{.5pt}
\medmuskip = 1mu
\frac{\begin{vmatrix}\begin{bmatrix}\left(\frac{1}{k}\right)\sum_{l=1}^k1 & \left(\frac{1}{k}\right)\sum_{l=1}^kt_l & \cdots & \left(\frac{1}{k}\right)\sum_{l=1}^kt_l^i\\
\left(\frac{1}{k}\right)\sum_{l=1}^kt_l & \left(\frac{1}{k}\right)\sum_{l=1}^kt_l^2 & \cdots & \left(\frac{1}{k}\right)\sum_{l=1}^kt_l^{i+1}\\
\vdots & \vdots & \ddots & \vdots\\
\left(\frac{1}{k}\right)\sum_{l=1}^kt_l^{i-1} & \left(\frac{1}{k}\right)\sum_{l=1}^kt_l^i & \cdots & \left(\frac{1}{k}\right)\sum_{l=1}^kt_l^{2i-1}\\
1 & t_j & \cdots & t_j^i\end{bmatrix}\begin{bmatrix}\left(\frac{1}{k}\right)\sum_{l=1}^k1 & \left(\frac{1}{k}\right)\sum_{l=1}^kt_l & \cdots & \left(\frac{1}{k}\right)\sum_{l=1}^kt_l^i\\
\left(\frac{1}{k}\right)\sum_{l=1}^kt_l & \left(\frac{1}{k}\right)\sum_{l=1}^kt_l^2 & \cdots & \left(\frac{1}{k}\right)\sum_{l=1}^kt_l^{i+1}\\
\vdots & \vdots & \ddots & \vdots\\
\left(\frac{1}{k}\right)\sum_{l=1}^kt_l^{i-1} & \left(\frac{1}{k}\right)\sum_{l=1}^kt_l^i & \cdots & \left(\frac{1}{k}\right)\sum_{l=1}^kt_l^{2i-1}\\
1 & t_m & \cdots & t_m^i\end{bmatrix}\end{vmatrix}}
{\begin{vmatrix}\left(\frac{1}{k}\right)\sum_{l=1}^k1 & \left(\frac{1}{k}\right)\sum_{l=1}^kt_l & \cdots & \left(\frac{1}{k}\right)\sum_{l=1}^kt_l^i\\
\left(\frac{1}{k}\right)\sum_{l=1}^kt_l & \left(\frac{1}{k}\right)\sum_{l=1}^kt_l^2 & \cdots & \left(\frac{1}{k}\right)\sum_{l=1}^kt_l^{i+1}\\
\vdots & \vdots & \ddots & \vdots\\
\left(\frac{1}{k}\right)\sum_{l=1}^kt_l^i & \left(\frac{1}{k}\right)\sum_{l=1}^kt_l^{i+1} & \cdots & \left(\frac{1}{k}\right)\sum_{l=1}^kt_l^{2i}\end{vmatrix}
\begin{vmatrix}\left(\frac{1}{k}\right)\sum_{l=1}^k1 & \left(\frac{1}{k}\right)\sum_{l=1}^kt_l & \cdots & \left(\frac{1}{k}\right)\sum_{l=1}^kt_l^{i-1}\\
\vdots & \vdots & \ddots & \vdots\\
\left(\frac{1}{k}\right)\sum_{l=1}^kt_l^{i-1} & \left(\frac{1}{k}\right)\sum_{l=1}^kt_l^i & \cdots & \left(\frac{1}{k}\right)\sum_{l=1}^kt_l^{2i-2}\end{vmatrix}}$};
\end{equation*}
then (ii) follows once we show that $C_{j,j}=k$ for $j=1,\ldots,k$. To see this, we note that $C=A^TA$, where
\begin{equation*}
A_{j,m}=
\scalebox{0.85}{$\tiny
\setlength{\arraycolsep}{.5pt}
\medmuskip = 1mu
\frac{\begin{vmatrix}\left(\frac{1}{k}\right)\sum_{l=1}^k1 & \left(\frac{1}{k}\right)\sum_{l=1}^kt_l & \cdots & \left(\frac{1}{k}\right)\sum_{l=1}^kt_l^j\\
\left(\frac{1}{k}\right)\sum_{l=1}^kt_l & \left(\frac{1}{k}\right)\sum_{l=1}^kt_l^2 & \cdots & \left(\frac{1}{k}\right)\sum_{l=1}^kt_l^{j+1}\\
\vdots & \vdots & \ddots & \vdots\\
\left(\frac{1}{k}\right)\sum_{l=1}^kt_l^{j-1} & \left(\frac{1}{k}\right)\sum_{l=1}^kt_l^j & \cdots & \left(\frac{1}{k}\right)\sum_{l=1}^kt_l^{2j-1}\\
1 & t_m & \cdots & t_m^i\end{vmatrix}}
{\left({\begin{vmatrix}\left(\frac{1}{k}\right)\sum_{l=1}^k1 & \left(\frac{1}{k}\right)\sum_{l=1}^kt_l & \cdots & \left(\frac{1}{k}\right)\sum_{l=1}^kt_l^j\\
\left(\frac{1}{k}\right)\sum_{l=1}^kt_l & \left(\frac{1}{k}\right)\sum_{l=1}^kt_l^2 & \cdots & \left(\frac{1}{k}\right)\sum_{l=1}^kt_l^{j+1}\\
\vdots & \vdots & \ddots & \vdots\\
\left(\frac{1}{k}\right)\sum_{l=1}^kt_l^j & \left(\frac{1}{k}\right)\sum_{l=1}^kt_l^{j+1} & \cdots & \left(\frac{1}{k}\right)\sum_{l=1}^kt_l^{2j}\end{vmatrix}
\begin{vmatrix}\left(\frac{1}{k}\right)\sum_{l=1}^k1 & \left(\frac{1}{k}\right)\sum_{l=1}^kt_l & \cdots & \left(\frac{1}{k}\right)\sum_{l=1}^kt_l^{j-1}\\
\vdots & \vdots & \ddots & \vdots\\
\left(\frac{1}{k}\right)\sum_{l=1}^kt_l^{j-1} & \left(\frac{1}{k}\right)\sum_{l=1}^kt_l^j & \cdots & \left(\frac{1}{k}\right)\sum_{l=1}^kt_l^{2j-2}\end{vmatrix}}\right)^{1/2}}$}.
\end{equation*}
On the other hand, we see that
\begin{align*}
(AA^T)_{j,m}&=\sum_{i=1}^kP_j(t_i)P_m(t_i)\\
&=k\delta_{j,m},
\end{align*}
with the last equality by Equation~\eqref{eq:kroneckerprop}. Hence, $AA^T=kI_{k\times k}$, and therefore we have $C=A^TA=kI_{k\times k}$. Thus, $C_{j,j}=k$ for $j=1,\ldots,k$, as desired.
\end{proof}
Since the sum of the $r$th powers of all the eigenvalues of a matrix is equivalent to the trace of its $r$th power, Proposition~\ref{prop:momentmatching} implies the following statement:
\begin{corollary}
\label{cor:momentmatching}
Suppose $A\in\mathbb{R}^{n\times n}$ and $B\in\mathbb{R}^{k\times k}$ have distinct, nonnegative eigenvalues, and suppose we have $\mathrm{tr}\left(B^r/k\right)=\mathrm{tr}\left(A^r/n\right)$ for $1\leq r\leq k$. Then $\sigma_{\lceil\frac{jk}{n}\rceil-1}(B)\geq\sigma_j(A)\geq\sigma_{\lceil\frac{jk}{n}\rceil+1}(B)$ for $1\leq j\leq n$, where we define ``$\sigma_0(B)=\infty$" and ``$\sigma_{k+1}(B)=0$."
\end{corollary}
\begin{proof}
Define the empirical spectral distributions $\mathcal{A}=\{\sigma_1(A)\,\ldots,\sigma_n(A)\}$ and $\mathcal{B}=\{\sigma_1(B),
\ldots,\sigma_k(B)\}$ as above, and apply Proposition~\ref{prop:momentmatching} while setting $S=\mathcal{A}$ and $T=\mathcal{B}$. The result then follows since $\mathrm{tr}(A^r)=\sum_{i=1}^n(\sigma_i(A))^r$ and $\mathrm{tr}(B^s)=\sum_{j=1}^k(\sigma_j(B))^s$ for $1\leq r\leq n$ and $1\leq s\leq k$, which follows, in turn, because $A$ and $B$ are positive-definite.
\end{proof}

\subsection{Matching traces in expectation} Given $A=\kappa(X,X)$, finding $B$ such that Equation~\eqref{eq:interlacing} holds requires us to match the traces of the $r$th powers of $A$ and $B$ for $r=1,\ldots,k$. Since $A$ and $B$ are random matrices, we will concentrate on understanding the expected traces of $A^r$ and $B^r$. In practice, we assume that $n$ is large enough such that $\sigma_i(A)$ does not vary very much from its expected value in relative terms, and such that sampling from $X$ is the same as sampling from the $n$ samples that formed $A$. However, since $k$ is small, we would need to form $B$ repeatedly $m$ times, where $m$ depends on $k$ and the desired approximation accuracy, and empirically compute the average value of $\sigma_i(B)$. These $\sigma_i(B)s$ would then be used in the way of \eqref{eq:interlacing}.

While we do not know of a way of matching these expected traces exactly, in the next proposition we show a way of matching them approximately if (a) $\kappa$ is ``close to the Kronecker delta"; that is, if $\kappa$ has very quick decay away from the diagonal; and (b) we have access to a special probability distribution $E$ on $\mathbb{R}$. As we will show, $E$ is picked to match walk statistics of the subsampled matrix $B$ (interpreted as a weighted graph) with those of the original matrix $A$. More precisely, $\kappa$ must satisfy the condition of Equation~\eqref{eq:quickdecay} below for some $\epsilon>0$ to give the relative moment bound \eqref{eq:approxexpect}, and $E$ must satisfy \eqref{eq:specialdistrib}. See Figure~\ref{fig:quickdecay} for an illustration of the condition on $\kappa$.

\begin{figure}[ptbh]
\centering
\begin{minipage}{0.9\textwidth}
\centering
\begin{tikzpicture}
  \node (img)  {\includegraphics[height=1.5in]{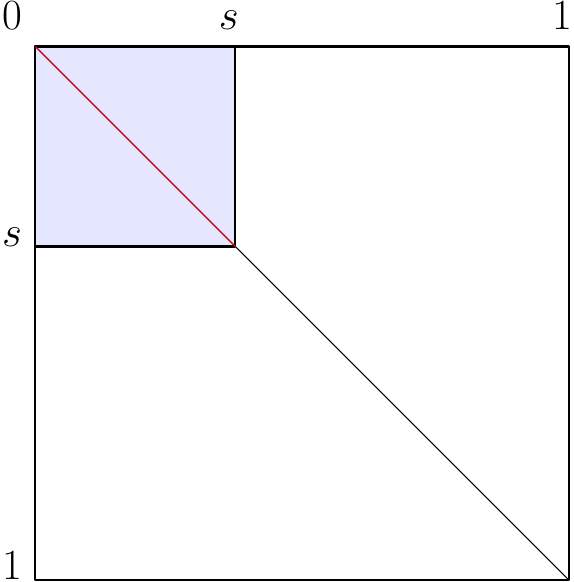}};
  \node[below=of img, node distance=0cm, yshift=1cm] {Distance from diagonal};
  \node[left=of img, node distance=0cm, rotate=90, anchor=center,yshift=-0.7cm] {Distance from diagonal};
\end{tikzpicture}\qquad
\begin{tikzpicture}
  \node (img)  {\includegraphics[height=1.7in,width=1.8in]{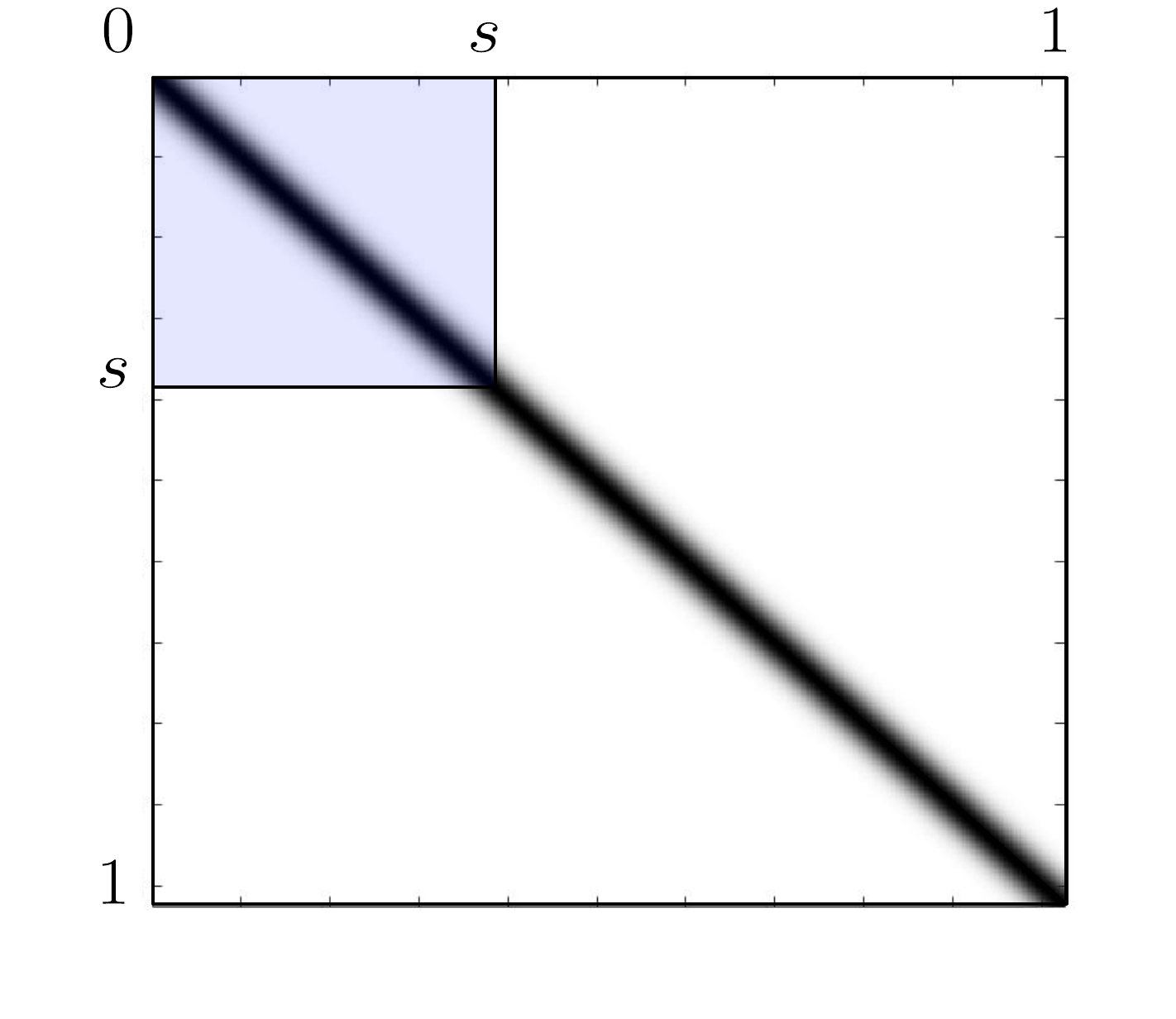}};
  \node[below=of img, node distance=0cm, yshift=1.48cm] {Distance from diagonal};
  \node[left=of img, node distance=0cm, rotate=90, anchor=center,yshift=-1.05cm] {Distance from diagonal};
 \end{tikzpicture}
\end{minipage}
\caption{The condition in \eqref{eq:quickdecay}: the left figure is a heatmap of the Kronecker delta on the region $[0,1]\times[0,1]$, and the right figure is a heatmap of the Gaussian kernel $\kappa_1(x,y)=e^{-1000(x-y)^2}$ on the same region. Informally, we may think of the integral of the Kronecker delta over the blue subregion $[0,s]\times[0,s]$ (the length of the red diagonal) as $s$ times its integral over the entire region $[0,1]\times[0,1]$ (the length of the entire diagonal). Of course, both integrals are formally 0. Similarly, we can see that the integral of $\kappa_1$ over $[0,s]\times[0,s]$ is approximately $s$ times its integral over $[0,1]\times[0,1]$. This is contrasted with the case, for example, of the Gaussian kernel $\kappa_2(x,y)=e^{-(x-y)^2/10000}$, whose integral over $[0,s]\times[0,s]$ is approximately $s^2$ times its integral over $[0,1]\times[0,1]$. Thus, the condition \eqref{eq:quickdecay} makes precise the way in which $\kappa_1$ does and $\kappa_2$ does not have fast decay away from the diagonal.}
\label{fig:quickdecay}
\end{figure}

In general, as we see in \cite{jinmusco}, approximate moment matching for a guarantee of pointwise closeness of two cumulative distribution functions may require prohibitively close tolerances. This is likely the main theoretical reason for the requirement that $\kappa$ decays quickly away from the diagonal. We will look at some examples in Section~\ref{sec:numerical} of how quick the decay has to be in practice for \eqref{eq:interlacing}. The recent work of \cite[Theorem 1]{muscomusco} suggests that we may, however, be able to bound our approximate quantile estimates in the Wasserstein-1 metric by a perturbative bound from the ``true" quantile estimate.

Once we have (a) and (b), we pick $B$ such that $\mathbb{E}(\mathrm{tr}(A^r)/n)=\mathbb{E}(\mathrm{tr}(B^r)/k)$ for all $r=1,\ldots,k$ using the following strategy:
\begin{enumerate}
\item we pick a set $Y$ of some points $\mathbf{y}_1,\ldots,\mathbf{y}_k$ at random from $X$;
\item we scale each $\mathbf{y}_i$ by a random number $z$ from a distribution $E$ satisfying \eqref{eq:specialdistrib}; and
\item we set $B=\kappa(Y,Y)$.
\end{enumerate}
We then repeat these steps $m$ times to find the average $\sigma_j(B)$ for $j=1,\ldots,k$.

In order to show why this graph-theoretic approach may work, we first need to fix notation for a walk on the complete graph on $n$ vertices $K_n$. We identify a function $\pi:\{0,\ldots,r\}\to\{1,\ldots,n\}$ with a walk of length $r$ starting (and ending) at a vertex $m$ of the complete graph $K_n$, where the value of $\pi(i)$ is the index of the vertex of $K_n$ visited at the $i$th step. In particular, note that since $\pi$ is a walk, $\pi(0)=\pi(r)=m$. We denote by $|\pi|$ the cardinality of the image of $\pi$. Then we have the following proposition:

\begin{proposition}
\label{prop:momentmatch}
Let $D,k,n\in\mathbb{N}$ with $k\mid n$. Suppose $z$ is sampled from a real-valued distribution $E$ such that
\begin{equation}
\label{eq:specialdistrib}
\mathbb{E}(z^j)=\frac{k}{n}\frac{\binom{n}{j+1}}{\binom{k}{j+1}}
\end{equation}
for all $1\leq j\leq k-1$. Define $\mathbf{x}_i\sim U[0,1]^D$ and $\mathbf{y}_j\sim(1/z)^{1/D}U[0,1]^D$ for $1\leq i\leq n$ and $1\leq j\leq k$; define $X=\{x_1,\ldots,y_n\}$ and $Y=\{y_1,\ldots,y_k\}$; set $A=\kappa(X,X)$ and $B=\kappa(Y,Y)$; and suppose that $\kappa~:~\mathbb{R}^D\times\mathbb{R}^D\to\mathbb{R}$ is a positive-definite function such that, and any walk $\pi$ with $|\pi|=l\leq n$ on $K_n$,
\begin{equation}
\label{eq:quickdecay}
\left\vert\frac{\int_{[0,s]^l}\prod_{i=1}^l\kappa(\mathbf{x}_{\pi(i-1)},\mathbf{x}_{\pi(i)})d\mathbf{x_\pi}}{\int_{[0,1]^l}\prod_{i=1}^l\kappa(\mathbf{x}_{\pi(i-1)},\mathbf{x}_{\pi(i)})d\mathbf{x_\pi}}-s\right\vert<\epsilon,
\end{equation}
for all $s\in[t,1]$, for some $0<t<1$.  Then
\begin{equation}
\label{eq:approxexpect}
\left\vert\frac{\mathbb{E}\left(\mathrm{tr}\left(B^r/k\right)\right)}{\mathbb{E}\left(\mathrm{tr}\left(A^r/n\right)\right)}-1\right\vert\leq\epsilon
\end{equation}
for $r=1,\ldots,k$.
\end{proposition}
\begin{proof}
First, note that 
\begin{equation*}
(A^r)_{mm}=\sum_{l=1}^r\sum_{\pi}\prod_{i=1}^rA_{\pi(i-1)\pi(i)},
\end{equation*}
where the inner sum ranges over all walks $\pi$ of length $r$ that visit $l$ distinct vertices on the complete graph $K_n$, starting at the vertex labeled $m$. Denote the set of all such walks, starting at any vertex, by $W_l^r(K_n)$. This bookkeeping of walks will be important for our argument to follow. Similarly, we have
\begin{equation*}
(B^r)_{mm}=\sum_{l=1}^r\sum_{\psi_l^m}\prod_{i=1}^rB_{\psi(i-1)\psi(i)},
\end{equation*}
where the inner sum ranges over all walks $\psi_m$ of length $r$ that visit $l$ distinct vertices on the complete graph $K_k$, starting at the vertex labeled $m$. Again, denote the set of all such walks, starting at any vertex, by $W_l^r(K_k)$.

Now, note that by linearity of expectation,
\begin{align*}
\mathbb{E}\left(\mathrm{tr}(A^r)\right)&=\sum_{m=1}^n\mathbb{E}\left((A^r)_{mm}\right)\\
&=\sum_{m=1}^n\mathbb{E}\left(\sum_{l=1}^r\sum_{\pi\in W^r_l(K_n)}\prod_{i=1}^rA_{\pi(i-1)\pi(i)}\right)\\
&=\sum_{l=1}^r\sum_{\pi\in W^r_l(K_n)}\mathbb{E}\left(\prod_{i=1}^rA_{\pi(i-1)\pi(i)}\right).
\end{align*}
By the definition of expectation and the variables $A_{ij}$, for each $\pi\in W^r_l(K_n)$, we have
\begin{align*}
\mathbb{E}\left(\prod_{i=1}^rA_{\pi(i-1)\pi(i)}\right)&=\int_{\mathbb{R}^l}\prod_{i=1}^r\kappa(\mathbf{x}_{\pi(i-1)},\mathbf{x}_{\pi(i)})f_\pi(\mathbf{x_{S_\pi}})d\mathbf{x_{S_\pi}}\\
&=\int_{[0,1]^l}\prod_{i=1}^r\kappa(\mathbf{x}_{\pi(i-1)},\mathbf{x}_{\pi(i)})d\mathbf{x_{S_\pi}},
\end{align*}
where $S_\pi$ is the set of vertices visited on the walk $\pi$ and $f_\pi$ is the probability density function of the joint distribution of the random variable $\mathbf{x_{S_\pi}}=(\mathbf{x}_{\pi(1)},\ldots,\mathbf{x}_{\pi(r)})$. Similarly, defining the $g_\psi$ to be the density of the variable $\mathbf{y_{S_\psi}}=(\mathbf{y}_{\psi(1)},\ldots,\mathbf{y}_{\psi(l)})$ for each $\psi\in W^r_l(K_k)$, we have
\begin{align}
\label{eq:propSummandType}
\mathbb{E}\left(\prod_{i=1}^rB_{\psi(i-1),\psi(i)}\right)&=\int_{\mathbb{R}^k}\int_{\mathbb{R}^l}\prod_{i=1}^r\kappa(\mathbf{y}_{\psi(i-1)},\mathbf{y}_{\psi(i)})g_\psi(\mathbf{y_{S_\psi}})d\mathbf{y_{S_\psi}}dz\nonumber\\
&=\int_{\mathbb{R}^k}\int_{\mathbb{R}^l}\prod_{i=1}^r\kappa(\mathbf{y}_{\psi(i-1)},\mathbf{y}_{\psi(i)})f_\psi(F_\psi^{-1}(\mathbf{y_{S_\psi}}))|\mathrm{Jac}(F_\psi^{-1})|d\mathbf{y_{S_\psi}}dz\nonumber\\
&=\int_{\mathbb{R}^k}\int_{[0,1/z]^l}\prod_{i=1}^r\kappa(\mathbf{y}_{\psi(i-1)},\mathbf{y}_{\psi(i)})z^ld\mathbf{y_{S_\psi}}dz\nonumber,
\end{align}
where $F_\psi:\mathbb{R}^l\to\mathbb{R}^l$ is the projection onto the indices $S_\psi$ of the function defined by $F(\mathbf{x})=(1/z,\ldots,1/z)\cdot\mathbf{x}$, restricted to the indices $S_\psi$. Note that the second equality follows from the change-of-variables formula for probability density functions applied to the variable $\mathbf{y_{S_\psi}}$, and the third equality follows from the definition of $F$ and the fact that $f_{\psi}=1$ for every $\psi$.

Hence, we see that
\begin{align*}
\frac{\mathbb{E}(\mathrm{tr}(B^r))}{\mathbb{E}(\mathrm{tr}(A^r))}&=\frac{\sum_{l=1}^r\sum_{\psi\in W^r_l(K_k)}\int_{\mathbb{R}^k}\int_{[0,1/z]^l}\prod_{i=1}^r\kappa(\mathbf{y}_{\psi(i-1)},\mathbf{y}_{\psi(i)})z^ld\mathbf{y_{S_\psi}}dz}{\sum_{l=1}^r\sum_{\pi\in W^r_l(K_n)}\int_{[0,1]^l}\prod_{i=1}^r\kappa(\mathbf{x}_{\pi(i-1)},\mathbf{x}_{\pi(i)})d\mathbf{x_{S_\pi}}}\\
&=\frac{\sum_{l=1}^r\frac{\binom{k}{l}}{\binom{n}{l}}\sum_{\pi\in W^r_l(K_n)}\int_{\mathbb{R}^k}\int_{[0,1/z]^l}\prod_{i=1}^r\kappa(\mathbf{y}_{\pi(i-1)},\mathbf{y}_{\pi(i)})z^ld\mathbf{y_{S_\pi}}dz}{\sum_{l=1}^r\sum_{\pi\in W^r_l(K_n)}\int_{[0,1]^l}\prod_{i=1}^r\kappa(\mathbf{x}_{\pi(i-1)},\mathbf{x}_{\pi(i)})d\mathbf{x_{S_\pi}}},
\end{align*}
where the second equality follows from the fact that, for every walk of length $r$ with $1\leq r\leq k$ visiting $l$ distinct vertices on $K_k$, there are $\binom{n}{l}/\binom{k}{l}$ such walks on $K_n$. Then, by our assumption on $\kappa$ in Equation~\eqref{eq:quickdecay},
\begin{align*}
(1-\epsilon)\frac{k}{n}&=\frac{(1-\epsilon)\left(\sum_{l=1}^r\frac{\binom{k}{l}}{\binom{n}{l}}\sum_{\pi\in W^r_l(K_n)}\int_{[0,1]^l}\prod_{i=1}^r\kappa(\mathbf{y}_{\pi(i-1)},\mathbf{y}_{\pi(i)})d\mathbf{y_{S_\pi}}\left(\frac{k}{n}\frac{\binom{n}{l}}{\binom{k}{l}}\right)\right)}{\sum_{l=1}^r\sum_{\pi\in W^r_l(K_n)}\int_{[0,1]^l}\prod_{i=1}^r\kappa(\mathbf{x}_{\pi(i-1)},\mathbf{x}_{\pi(i)})d\mathbf{x_{S_\pi}}}\\
&=\frac{(1-\epsilon)\sum_{l=1}^r\frac{\binom{k}{l}}{\binom{n}{l}}\sum_{\pi\in W^r_l(K_n)}\int_{[0,1]^l}\prod_{i=1}^r\kappa(\mathbf{y}_{\pi(i-1)},\mathbf{y}_{\pi(i)})d\mathbf{y_{S_\pi}}\mathbb{E}(z^{l-1})}{\sum_{l=1}^r\sum_{\pi\in W^r_l(K_n)}\int_{[0,1]^l}\prod_{i=1}^r\kappa(\mathbf{x}_{\pi(i-1)},\mathbf{x}_{\pi(i)})d\mathbf{x_{S_\pi}}}\\
&=\frac{\sum_{l=1}^r\sum_{\pi\in W^r_l(K_n)}\frac{\binom{k}{l}}{\binom{n}{l}}\int_{\mathbb{R}^k}\frac{1-\epsilon}{z}\int_{[0,1]^l}\prod_{i=1}^r\kappa(\mathbf{y}_{\pi(i-1)},\mathbf{y}_{\pi(i)})z^ld\mathbf{y_{S_\pi}}dz}{\sum_{l=1}^r\sum_{\pi\in W^r_l(K_n)}\int_{[0,1]^l}\prod_{i=1}^r\kappa(\mathbf{x}_{\pi(i-1)},\mathbf{x}_{\pi(i)})d\mathbf{x_{S_\pi}}}\\
&\leq\frac{\sum_{l=1}^r\sum_{\pi\in W^r_l(K_n)}\frac{\binom{k}{l}}{\binom{n}{l}}\int_{\mathbb{R}^k}\int_{[0,1/z]^l}\prod_{i=1}^r\kappa(\mathbf{y}_{\pi(i-1)},\mathbf{y}_{\pi(i)})z^ld\mathbf{y_{S_\pi}}dz}{\sum_{l=1}^r\sum_{\pi\in W^r_l(K_n)}\int_{[0,1]^l}\prod_{i=1}^r\kappa(\mathbf{x}_{\pi(i-1)},\mathbf{x}_{\pi(i)})d\mathbf{x_{S_\pi}}}\\
&=\frac{\mathbb{E}(\mathrm{tr}(B^r))}{\mathbb{E}(\mathrm{tr}(A^r))}.
\end{align*}
By linearity of trace and expectation, we thus get $1-\epsilon\leq\mathbb{E}(\mathrm{tr}(B^r/k))/\mathbb{E}(\mathrm{tr}(A^r/n))$. The other inequality comprising Equation~\eqref{eq:approxexpect} follows from Equation~\eqref{eq:quickdecay} in a similar way.
\end{proof}

Two questions immediately arise from this last proposition. First, it is not clear which functions $\kappa$ satisfy Equation~\eqref{eq:quickdecay}. We explore this topic empirically in Section~\ref{sec:numerical}. For the Gaussian kernel $\kappa:\mathbb{R}\times\mathbb{R}\to\mathbb{R}$ defined by $\kappa(x,y)=e^{-\lambda|x-y|^2}$, in particular, we note that for each $\pi\in W^r_l(K_n)$ and $s\in(0,1]$,
\begin{equation*}
\lim_{\lambda\to\infty}\frac{\int_{[0,s]^l}e^{-\lambda\sum_{i=1}^r|x_{\pi(i-1)}-x_{\pi(i)}|^2}d\mathbf{x_{S_\pi}}}{\int_{[0,1]^l}e^{-\lambda\sum_{i=1}^r|x_{\pi(i-1)}-x_{\pi(i)}|^2}d\mathbf{x_{S_\pi}}}=s.
\end{equation*}
Hence, for each $\epsilon>0$, there exists a length scale $\lambda$ that makes $\kappa$ satisfy \eqref{eq:quickdecay}. Analogous results may be obtained for other radial basis function (RBF) kernels by finding appropriate limits with respect to the length scale (as with respect to $\lambda$ above). However, the exact relationship of $s$, $l$, and $\lambda$ in the previous display to a given tolerance $\epsilon$ as in \eqref{eq:quickdecay} warrants further study, since it may allow for a more precise formulation of moment bounds. This may be done in combination with studies similar to \cite{jinmusco,muscomusco}.

Second, it is not clear {\it a priori} whether or not any distribution $E$ that satisfies \eqref{eq:specialdistrib} in the above proposition exists, and if it does, where its support lies. If such a distribution exists, then the method outlined at the beginning of this section should work. It turns out that such a distribution does exist; we next give an example.

\begin{example}
\label{ex:samplingdistribution}Fix $n=49$, $k=7$. We construct a distribution $E$ satisfying the moment conditions of \eqref{eq:specialdistrib}. To do so, we assume that $E$ has finite support. This simplifying assumption makes the second equation of \eqref{eq:specialdistrib} equivalent to the system of 7 equations in 8 unknowns
\begin{align*}
a+b+c+d&=(k/n)\textstyle\binom{n}{1}/\binom{k}{1}=1\\
a\alpha+b\beta+c\gamma+d\delta&=(k/n)\textstyle\binom{n}{2}/\binom{k}{2}=8\\
a\alpha^2+b\beta^2+c\gamma^2+d\delta^2&=(k/n)\textstyle\binom{n}{3}/\binom{k}{3}=\frac{376}{5}\\
a\alpha^3+b\beta^3+c\gamma^3+d\delta^3&=(k/n)\textstyle\binom{n}{4}/\binom{k}{4}=\frac{4324}{5}\\
a\alpha^4+b\beta^4+c\gamma^4+d\delta^4&=(k/n)\textstyle\binom{n}{5}/\binom{k}{5}=12972\\
a\alpha^5+b\beta^5+c\gamma^5+d\delta^5&=(k/n)\textstyle\binom{n}{6}/\binom{k}{6}=285384\\
a\alpha^6+b\beta^6+c\gamma^6+d\delta^6&=(k/n)\textstyle\binom{n}{7}/\binom{k}{7}=12271512.
\end{align*}
More specifically, we assume that $z$ takes on four distinct values $\alpha,\beta,\gamma,\delta$; this is done to give $z$ enough degrees of freedom to satisfy \eqref{eq:specialdistrib}. (Otherwise, we would not have enough unknowns to satisfy the 7 equations above.) The values $a\approx0.41166$, $b\approx0.56810$, $c\approx0.020241$, $d\approx1.4709\cdot10^{-6}$, $\alpha\approx4.8651$, $\beta\approx9.6827$, $\gamma\approx24.519$, and $\delta\approx130.90$ form a solution to this system. Hence, taking $E$ to be the distribution that gives $\alpha,\beta,\gamma$, and $\delta$ with probabilities $a$, $b$, $c$, and $d$, respectively, we find that $E$ yields \eqref{eq:specialdistrib}. Note that this is equivalent to simply letting $Y$ be a random subset of points in $X$ scaled by $\alpha$, $\beta$, $\gamma$, and $\delta$, with probabilities $a$, $b$, $c$, and $d$, respectively.
\end{example}

We found a distribution $E$ in Example~\ref{ex:samplingdistribution} that we may use to build a matrix $B$ from $A$ such that \eqref{eq:approxexpect} holds, but only for the case that $n=49$ and $k=7$. To do so, we assumed $E$ is discrete, which yielded a straightforward system of polynomial equations we could use to find $E$ from the second equation of \eqref{eq:specialdistrib}. This construction naturally leads to two questions: first, can we use this technique to find such a distribution for every $n,k$ such that $k|n$? And second, will the support of such a distribution take values that are ``too large" to truncate $\kappa$ in such a manner as to make \eqref{eq:approxexpect} provide a meaningfully-small $\epsilon$? To answer these last two questions, we prove the following proposition. It states that we may always find a distribution with nonnegative support satisfying \eqref{eq:specialdistrib}, although further questions about its support may be harder to answer.

\begin{proposition}
\label{prop:distributionexistence}
Let $k,n\in\mathbb{N}$ such that $k$ is odd and $k|n$. There exists a distribution $E$ satisfying \eqref{eq:specialdistrib}.
\end{proposition}
\begin{proof}
As for the two specific values of $k$ and $n$ considered in Example~\ref{ex:samplingdistribution} above, the moment conditions are $k$ prescribed moment conditions for a univariate probability distribution $z$ with nonnegative support:
\begin{equation*}
\mathbb{E}(z^l)=\frac{\frac{k}{n}\binom{n}{l+1}}{\binom{k}{l+1}},\quad l=0,\ldots,k-1.
\end{equation*}
(Note that these moment conditions are distinct from the moment conditions we considered in Proposition~\ref{prop:momentmatching}.) But this is just the so-called Stieltjes moment problem, which is well-known to have a solution if certain moment matrices are positive semidefinite and full-rank (or, equivalently, positive definite). For a complete treatment of this question and questions on related moment problems, see the treatise of Curto and Fialkow on the subject \cite[Theorem~5.3]{curtofialkow}. From that result, we see that showing the Proposition comes down to showing that the Hankel matrices
\begin{align*}
H_{k,n}&=\begin{bmatrix}\mu_0 & \mu_1 & \cdots & \mu_{(k-1)/2}\\
\mu_1 & \mu_2 & \cdots & \mu_{(k-1)/2+1}\\
\vdots & \vdots & \ddots & \vdots\\
\mu_{(k-1)/2} & \mu_{(k-1)/2+1} & \cdots & \mu_{k-1}\end{bmatrix}\quad\mathrm{and}\\
H'_{k,n}&=\begin{bmatrix}\mu_1 & \mu_2 & \cdots & \mu_{(k-1)/2}\\
\mu_2 & \mu_3 & \cdots & \mu_{(k-1)/2+1}\\
\vdots & \vdots & \ddots & \vdots\\
\mu_{(k-1)/2} & \mu_{(k-1)/2+1} & \cdots & \mu_{k-1}\end{bmatrix}
\end{align*}
are positive definite, where $\mu_l={\frac{k}{n}\binom{n}{l+1}}/{\binom{k}{l+1}}$ for $l=0,\ldots,k-1$. We see this once we realize both $H_{k,n}$ and $H'_{k,n}$ as Gram matrices associated to linearly independent sets of vectors in a Hilbert space. In particular, consider the space $V$ of square-integrable functions on the compact interval $[0,1]$ with respect to the Radon-Nikodym derivative $x^{n-k-1}(1-x)^{k+1}$. For $i=0,\ldots,(k-1)/2$ define $v_i=\sqrt{n-k}(1/(1-x))^{i+1/2}$; and for $j=0,\ldots,(k-1)/2-1$, define $w_i=\sqrt{n-k}(1/(1-x))^{i+1}$. Clearly, we have $v_i,w_j\in V$ for $i=0,\ldots,(k-1)/2$ and $j=0,\ldots,(k-1)/2-1$. Furthermore, the sets $\{v_i\}_{i=0}^{(k-1)/2}$ and $\{w_j\}_{j=0}^{(k-1)/2-1}$ are linearly independent, and we see that
\begin{align*}
\frac{1}{\frac{k}{n}\binom{n}{k}}\mu_l&=\frac{1}{\frac{k}{n}\binom{n}{k}}\frac{\frac{k}{n}\binom{n}{l+1}}{\binom{k}{l+1}}\\
&=\frac{1}{\binom{n}{k}}\frac{\binom{n}{k}}{\binom{n-(l+1)}{k-(l+1)}}\\
&=\frac{(n-k)!(k-(l+1))!}{(n-(l+1))!}\\
&=(n-k)\int_0^1x^{n-(l+2)-(k-(l+1))}(1-x)^{k-l}dx\\
&=\int_0^1\left(\sqrt{n-k}\frac{1}{(1-x)^{i+1/2}}\right)\left(\sqrt{n-k}\frac{1}{(1-x)^{j+1/2}}\right)x^{n-k-1}(1-x)^{k+1}dx\\
&=\langle v_i,v_j\rangle_V.
\end{align*}
whenever $i+j=l$ for $l=0,\ldots,k-1$. Hence, the Gram matrix $H_{k,n}$ associated to $\{v_i\}_{i=0}^{(k-1)/2}$ in $V$ is positive definite. Similarly,
\begin{align*}
\frac{1}{\frac{k}{n}\binom{n}{k}}\mu_{l+1}&=\int_0^1\left(\sqrt{n-k}\frac{1}{(1-x)^{i+1}}\right)\left(\sqrt{n-k}\frac{1}{(1-x)^{j+1}}\right)x^{n-k-1}(1-x)^{k+1}dx\\
&=\langle w_i,w_j\rangle_V
\end{align*}
whenever $i+j=l$ for $l=0,\ldots,k-2$, so $H'_{k,n}$ associated to $\{w_j\}_{j=0}^{(k-1)/2-1}$ in $V$ is also positive definite.
\end{proof}
Here, we note two things: first, we assumed $k$ is odd in showing the existence of $E$. The case when $k$ is even is handled similarly, so we omit it for brevity. The main theoretical difference is that we use Theorem~5.1 of \cite{curtofialkow} (and therefore that the distribution $E$ thus obtained is actually unique, but that is irrelevant for our examples) instead of Theorem~5.3. Second, computing a distribution as in Example~\ref{ex:samplingdistribution} may be no small task for large values of $k$ and may take a lot of computing power. Nevertheless, since $E$ does not depend on the specific choice of $\kappa$ as long as $\kappa$ satisfies the condition of Equation~\eqref{eq:quickdecay}, we may precompute the values $E$ for each combination of values of $k,n$. This is the ``preprocessing step" alluded to in the introduction.

\section{Numerical experiments}
\label{sec:numerical}
The last proposition thus completes an answer for how, given $X=\{\mathbf{x}_1,\ldots,\mathbf{x}_n\}$ with $\mathbf{x}_i\in U[0,1]^D$ for $1\leq i\leq n$ and $A=\kappa(X,X)$, we may design a framework for obtaining a matrix $B$ such that Corollary~\ref{cor:momentmatching} applies in expectation. Namely, we will fix $k$ and $n$, precompute $E$ as in Proposition~\ref{prop:distributionexistence} above, and then take $B=\kappa(Y,Y)$, where $Y=\{\mathbf{y}_1,\ldots,\mathbf{y}_k\}$ is defined as in Proposition~\ref{prop:momentmatch} using the distribution of Proposition~\ref{prop:distributionexistence}. That is, $Y$ is the set obtained by multiplying a random subsample of $X$ by a random scalar picked using $E$. Because this way of obtaining $Y$ is probabilistic and only guarantees moment matching in expectation, we thus need to find the average of the $j$th largest eigenvalue of $B$, for $1\leq j\leq k$, for a number of trials $m$ of forming such matrices $B$. Even though $E$ depends on $n$ and $k$, empirically $m$ seems to depends on $k$ alone. The average $\sigma_j(B)$s should then correspond to bounds for the $k$ quantiles of the eigenvalues of $A$ as in \eqref{eq:interlacing}. First, we look at the performance of this framework using the distribution $E$ computed in Example~\ref{ex:samplingdistribution} (that is, we set $n=49$ and $k=7$).

\begin{example}
\label{ex:firstex}
Let $n=49$, $k=7$, $\kappa_1:\mathbb{R}^3\times\mathbb{R}^3\to\mathbb{R}$ be defined by $\kappa_1(\mathbf{x},\mathbf{y})=e^{-10000|\mathbf{x}-\mathbf{y}|}$. Since $n$ is so small in this case, we perform 100 trials of forming $A=\kappa(X,X)$ and average the $j$th largest eigenvalue for $1\leq j\leq n$. We then perform $m=100000$ trials of forming $B=\kappa(Y,Y)$ according to the scheme in Proposition~\ref{prop:momentmatch} using the distribution from Proposition~\ref{prop:distributionexistence}, and we average the $j$th largest eigenvalue thus obtained for $1\leq j\leq k$. The resulting averaged eigenvalues of $A$ are plotted in Figure~\ref{fig:firstex1}, along with the eigenvalue quantile bounds obtained from the averaged eigenvalues of $B$. (We repeat each eigenvalue of $B$ $49/7=7$ times in order to better visualize the quantile bounds given for the eigenvalues of $A$ by Corollary~\ref{cor:momentmatching}, as in Figure~\ref{fig:momentinterlace}.) This same setup is repeated in Figure~\ref{fig:firstex2} for $A=\kappa_2(X,X)$ and $B=\kappa_2(Y,Y)$, where $\kappa_2:\mathbb{R}^6\times\mathbb{R}^6\to\mathbb{R}$ is defined by $\kappa_2(\mathbf{x},\mathbf{y})=e^{-100|\mathbf{x}-\mathbf{y}|}$.
\end{example}

\begin{figure}[ptbh]
\centering
\begin{minipage}{0.65\textwidth}
\begin{tikzpicture}
  \node (img)  {\includegraphics[height=2.75in]{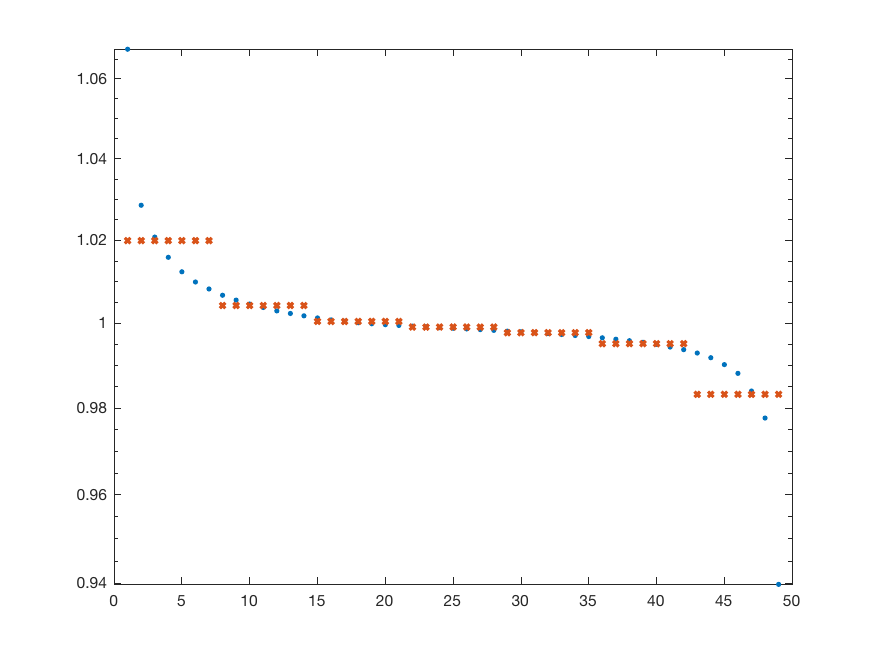}};
  \node[below=of img, node distance=0cm, yshift=1.1cm] {Index of eigenvalue};
  \node[left=of img, node distance=0cm, rotate=90, anchor=center,yshift=-0.7cm] {Magnitude};
 \end{tikzpicture}
\end{minipage}
\caption{The averaged eigenvalues of $A=\kappa_1(X,X)$ (blue dots) together with the repeated, averaged eigenvalues of $B=\kappa_1(Y,Y)$ (red crosses).}\label{fig:firstex1}
\end{figure}

\begin{figure}[ptbh]
\centering
\begin{minipage}{0.65\textwidth}
\begin{tikzpicture}
  \node (img)  {\includegraphics[height=2.75in]{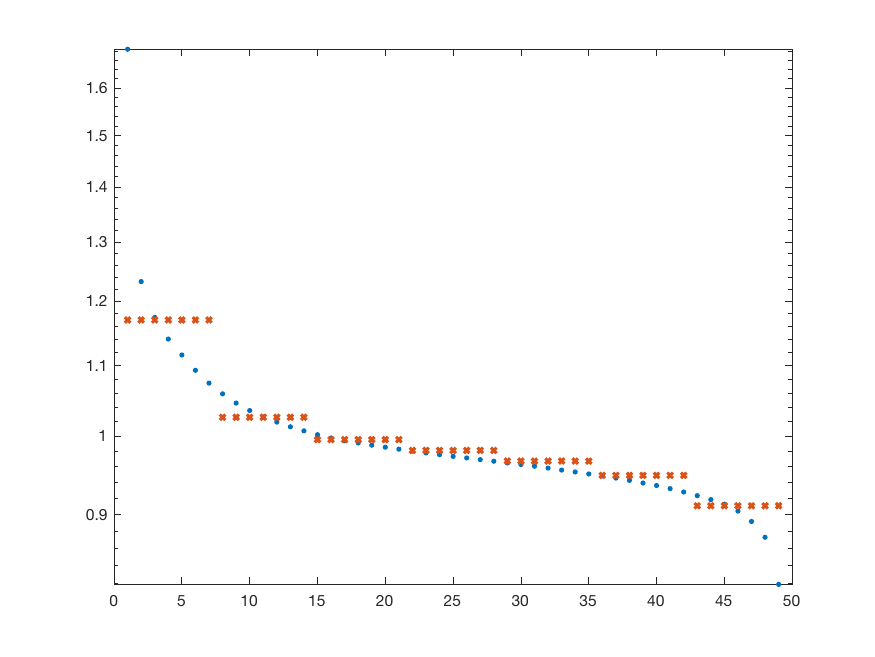}};
  \node[below=of img, node distance=0cm, yshift=1.1cm] {Index of eigenvalue};
  \node[left=of img, node distance=0cm, rotate=90, anchor=center,yshift=-0.7cm] {Magnitude};
 \end{tikzpicture}
\end{minipage}
\vspace{-.1cm}
\caption{The averaged eigenvalues of $A=\kappa_2(X,X)$ (blue dots) together with the repeated, averaged eigenvalues of $B=\kappa_2(Y,Y)$ (red crosses).}\label{fig:firstex2}
\vspace{-.5cm}
\end{figure}

Note the length scale of $\kappa$: setting $\kappa$ to have such quick decay away from the diagonal is necessary to have a meaningful correlation between the quantile bounds obtained from the eigenvalue distribution of $B$ and those from the eigenvalue distribution of $A$. We will see in Example~\ref{ex:thirdex} what happens with our framework if this is not the case.

\begin{example}
\label{ex:secondex}
We set $n=729$, $k=9$, and $\kappa_3:\mathbb{R}\times\mathbb{R}\to\mathbb{R}$ to be defined by $\kappa_3(\mathbf{x},\mathbf{y})=e^{-5000000(\mathbf{x}-\mathbf{y})^2}$. We form $A=\kappa(X,X)$ and plot its $n$ eigenvalues. We then perform $m=100000$ trials of forming $B=\kappa(Y,Y)$ from samples taken from those used to form $A$ and average the $j$th largest eigenvalue thus obtained for $1\leq j\leq k$. The eigenvalues of $A$ are plotted in Figure~\ref{fig:secondex1}, along with the eigenvalue quantile bounds obtained from the averaged eigenvalues of $B$. (As before, we repeat each eigenvalue of $B$ $729/9=81$ times in order to visualize the quantile bounds given for the eigenvalues of $A$ by Corollary~\ref{cor:momentmatching}.) This same setup is then repeated in Figure~\ref{fig:secondex2} for $A=\kappa_4(X,X)$ and $B=\kappa_4(Y,Y)$, where $\kappa_4:\mathbb{R}^8\times\mathbb{R}^8\to\mathbb{R}$ is defined by $\kappa_4(\mathbf{x},\mathbf{y})=e^{-20|\mathbf{x}-\mathbf{y}|^2}$.
\end{example}

\begin{figure}[ptbh]
\centering\vspace{-3.5cm}
\begin{minipage}{0.65\textwidth}
\begin{tikzpicture}
  \node (img)  {\includegraphics[width=4in]{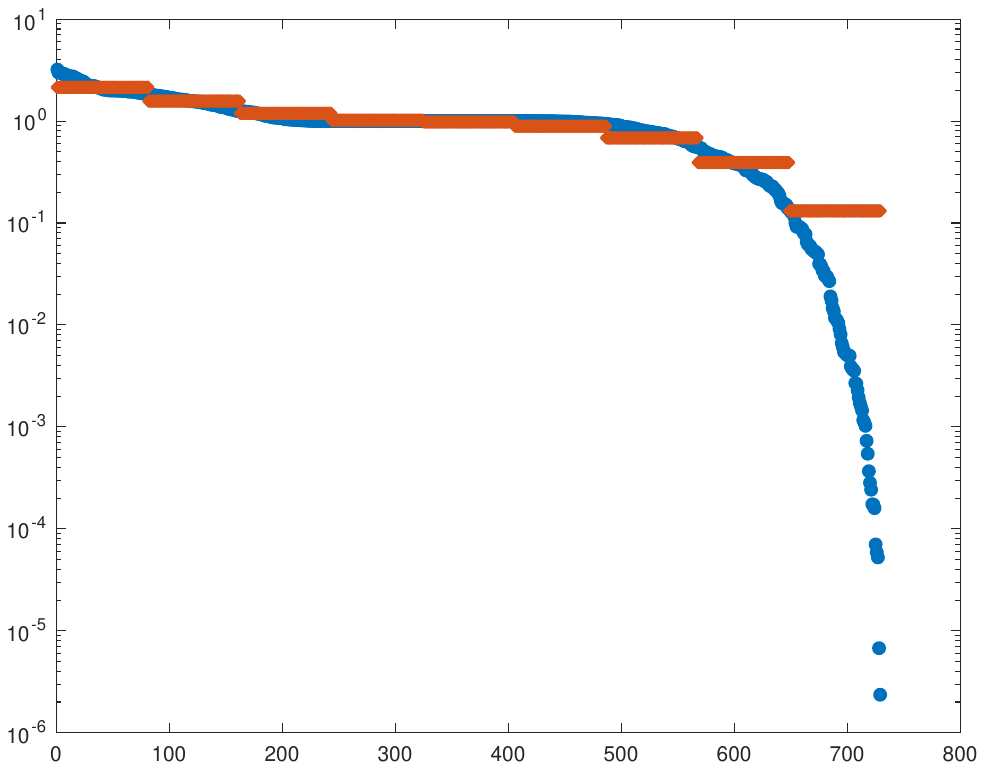}};
  \node[below=of img, node distance=0cm, yshift=4.2cm] {Index of eigenvalue};
  \node[left=of img, node distance=0cm, rotate=90, anchor=center,yshift=-1.15cm] {Magnitude};
 \end{tikzpicture}
\end{minipage}
\vspace{-2.8cm}
\caption{The averaged eigenvalues of $A=\kappa_3(X,X)$ (blue dots) together with the repeated, averaged eigenvalues of $B=\kappa_3(Y,Y)$ (red crosses).}\label{fig:secondex1}
\end{figure}

\begin{figure}[htbp]
\centering\vspace{-3.5cm}
\begin{minipage}{0.65\textwidth}
\begin{tikzpicture}
  \node (img)  {\includegraphics[width=4in]{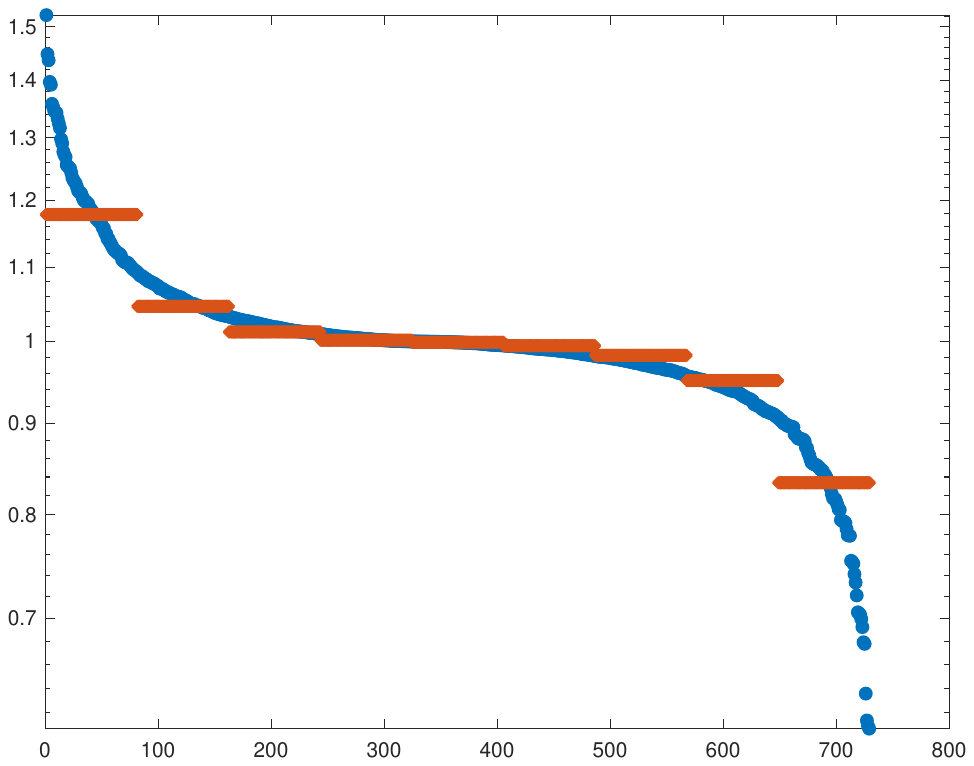}};
  \node[below=of img, node distance=0cm, yshift=4.2cm] {Index of eigenvalue};
  \node[left=of img, node distance=0cm, rotate=90, anchor=center,yshift=-1.15cm] {Magnitude};
 \end{tikzpicture}
\end{minipage}
\vspace{-2.8cm}
\caption{The averaged eigenvalues of $A=\kappa_4(X,X)$ (blue dots) together with the repeated, averaged eigenvalues of $B=\kappa_4(Y,Y)$ (red crosses).}\label{fig:secondex2}
\end{figure}

\begin{example}
\label{ex:thirdex}
Figure~\ref{fig:thirdex} shows what happens when the setup is kept exactly the same as in Example~\ref{ex:secondex}, except we set the kernel to be $\kappa_5:\mathbb{R}\times\mathbb{R}\to\mathbb{R}$ defined by $\kappa_5(x,y)=e^{-10(x-y)^2}$. Observe that there seems to be no correlation between the quantile bounds for the eigenvalues of $A$ and the averaged eigenvalues of $B$, which we may attribute to a lack of decay of $\kappa_5$ away from the diagonal as required by \eqref{eq:quickdecay}. (Note that $A$ has low numerical rank here.)
\begin{figure}[htbp]
\centering\vspace{-3.5cm}
\begin{minipage}{0.65\textwidth}
\begin{tikzpicture}
  \node (img)  {\includegraphics[width=4in]{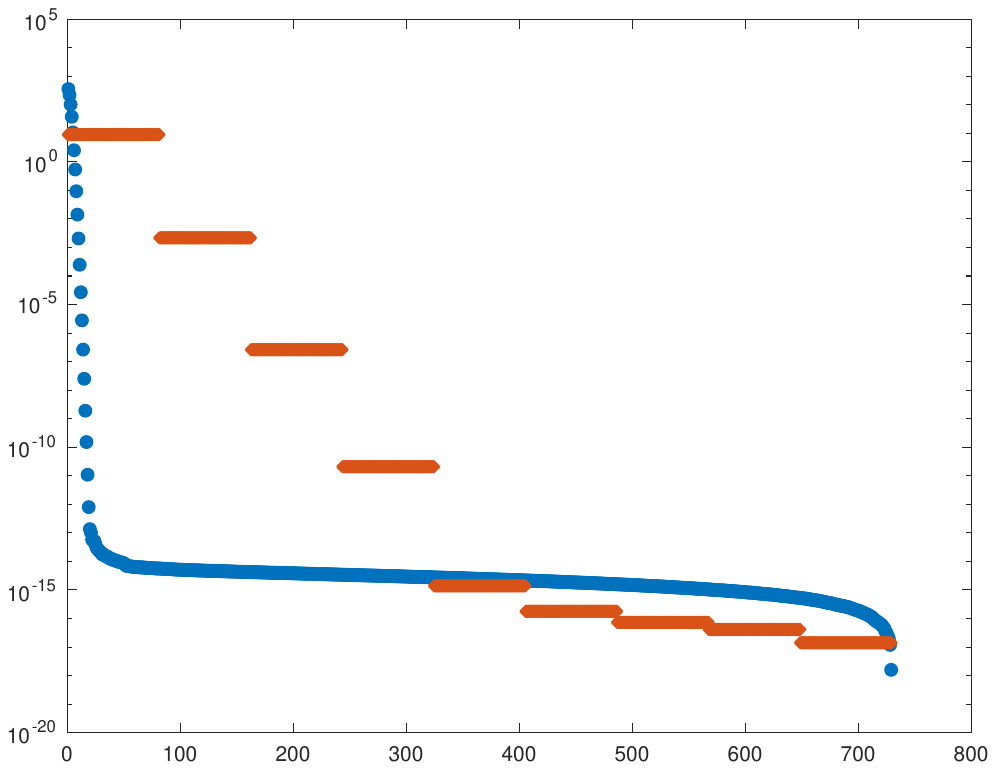}};
  \node[below=of img, node distance=0cm, yshift=4.2cm] {Index of eigenvalue};
  \node[left=of img, node distance=0cm, rotate=90, anchor=center,yshift=-1.15cm] {Magnitude};
 \end{tikzpicture}
\end{minipage}
\vspace{-2.8cm}
\caption{The averaged eigenvalues of $A=\kappa_5(X,X)$ (blue dots) together with the repeated, averaged eigenvalues of $B=\kappa_5(Y,Y)$ (red crosses).}\label{fig:thirdex}
\end{figure}
\end{example}

Here, we note that the kernel used does not have to have any particular form (i.e. we take $\kappa$ alternately to be the Cauchy kernel and the Gaussian kernel), as long as the steep decay away from the diagonal is maintained. For higher dimensions, Examples~\ref{ex:firstex} and \ref{ex:secondex} indicate that the length scale involved in $\kappa$ does not have to be quite as small in higher dimensions as in does in lower dimensions for fast decay to be satisfied. This corresponds to the well-known (but unintuitive) heuristic that unit balls in high dimension are ``concentrated near the axes." Finally, Example~\ref{ex:thirdex} illustrates the limitations of our framework; in a sense, it is a complement of the phenomenon that occurs in the bottom plot of Figure~\ref{fig:nystrom}.

\section{Conclusion and future work}
\label{sec:futurework}
We have introduced a new framework that aims to provide a way to approximate the eigenvalues of a kernel matrix $A$ evaluated at a set $X$ of $n$ points which come from standard uniform distributions on $\mathbb{R}^D$. In particular, after fixing $k$, our framework provides bounds in expectation on the $k$ spectrum quantiles of the kernel matrix $A$. Since we do not require access to $A$, this new framework allows us to find such bounds with computational cost that does not scale with $n$. In particular, it requires $O(mk^3)$ steps, where $m$ is the number of times we form $B$ and is a constant that appears to scale with $k$. However, our work includes a number of limitations that we aim to overcome in the future. We go over these limitations one by one, and mention possible directions to address them.

First, our work so far concerned only points which come from the uniform distribution on $\mathbb{R}^D$. However, we may extend these ideas to consider any compactly-supported, absolutely continuous distribution $\Omega$ by composing $\kappa$ with an appropriate coordinate transformation, which in turn may be obtained from the CDF of $\Omega$. In doing so, for our framework to apply, we must ensure that an analog of the condition of Equation~\eqref{eq:quickdecay} is adequately satisfied on this composition of functions, and we must account for this difference in the resulting analysis and design of the subsampling strategy. A future study of commonly-used distributions (for example, the multivariate normal distribution) will be useful in finding evidence for when this condition is satisfied.

Second, the distribution $E$ provided by Proposition~\ref{prop:distributionexistence} requires a lot of trials of forming, finding the eigenvalues of, and then averaging $B$ in order to get a good approximation for the quantiles of $A$. That is, $m$ may be high, if it does not depend on $n$. This is likely because the probabilities of some of the scalar multiples appear to be quite low in general. For example, in Example~\ref{ex:samplingdistribution}, we require each coordinate of $\mathbf{x}$ to be multiplied by $\delta=130.90$ with probability $d=1.4709\cdot10^{-6}$. Another disadvantage of $E$ from Proposition~\ref{prop:distributionexistence} is that precomputing the relevant values of $z$ and their probabilities is computationally expensive and becomes infeasible for large $k$. This distribution, however, is only one distribution that satisfies~\eqref{eq:specialdistrib}. We know from \cite{curtofialkow} that there is not even a unique discrete distribution satisfying Equation~\eqref{eq:specialdistrib}, for example if we assume that $E$ could take on more distinct values. Furthermore, there may be continuous distributions satisfying Equation~\eqref{eq:specialdistrib} that are easier to compute with for our purpose. We have only made the simplifying assumption that $E$ is discrete in order to demonstrate the theoretical feasibility of this approach. Thus, we would like to know if continuous distributions exist satisfying \eqref{eq:specialdistrib} that allow the quantile estimates to converge to their expectation with fewer trials than $E$ requires.

Finally, we may sidestep the graph-theoretic difficulties of matrix subsampling altogether by combining the quantile bound results contained in this work with moment information derived from the Stochastic Lanczos Quadrature approach of \cite{chen1,chen2}. This direction may allow the {\it a priori} determination of whether or not $A$ has high numerical rank, something that the framework we have outlined does not address.


In its present form, we propose applying this framework to the question of finding the intrinsic dimension of data. Namely, the {\it manifold hypothesis} in data science is that real-world data embedded in high-dimensional space, such as collections of 64-by-64-pixel images with certain properties embedded in the space of all 64-by-64-pixel images, actually reside on a lower-dimensional mathematical structure. The dimensionality of this structure is then called the intrinsic, or latent, dimension of the dataset. This idea, taken literally for the case of a $\mathcal{C}^2$-manifold, was tested in \cite{fefferman}. Taken literally for the case of an algebraic variety, it was tested in \cite{breiding}. Less literal but more practical mathematical formulations for intrinsic dimension have been explored starting with the case of linear subspace dimension \cite{fukunaga}. The most commonly used notions in practice are local, such as those based on statistics of the expected number of nearest neighbors at each datapoint \cite{levina,facco,pope}.

Here, we propose a new local formulation. Until now, we have not paid much attention to the role of $D$ used in the definition of the $\mathbf{y}_j$'s in Proposition~\ref{prop:momentmatch}. However, when the dimensionality of $X$ is unknown but the distribution of each point of $X$ is known to be approximately uniform, this constant may be used as a parameter. With these assumptions, only choosing $D$ correctly leads to obtaining good quantile bounds for a larger kernel matrix $A$ formed from $X$ versus a subsampled kernel matrix $B$. See, for example, Figure~\ref{fig:baddvalue} for the result of setting $D=2$ or $D=4$ instead of the correct value $D=3$ when forming $B$.

\begin{figure}[ptbh]
\centering
\begin{minipage}{0.6\textwidth}
\begin{tikzpicture}
  \node (img)  {\includegraphics[scale=0.57]{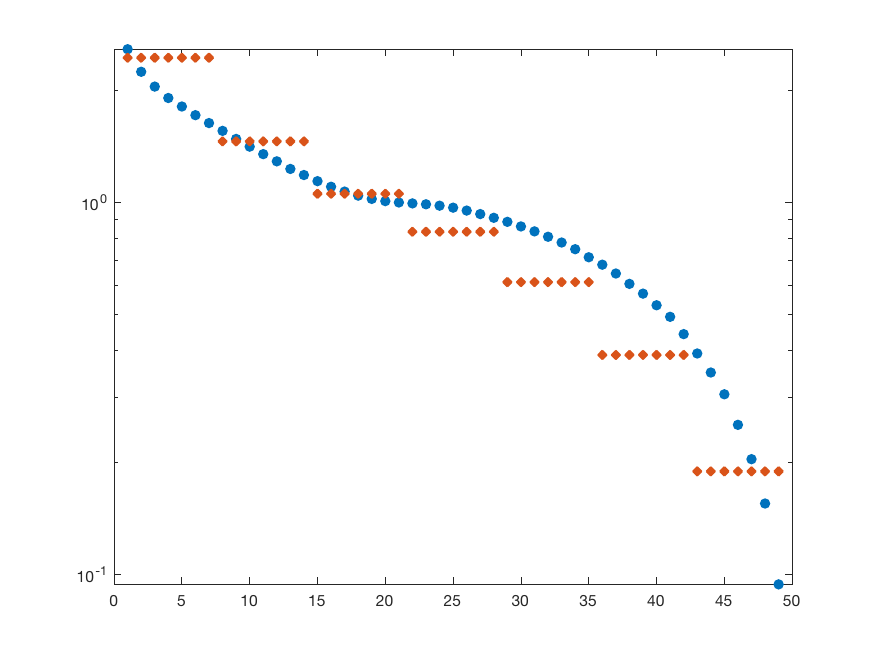}};
  \node[left=of img, node distance=0cm, rotate=90, anchor=center,yshift=-0.7cm] {Magnitude};
 \end{tikzpicture}
\end{minipage}

\begin{minipage}{0.6\textwidth}
\begin{tikzpicture}
  \node (img)  {\includegraphics[scale=0.57]{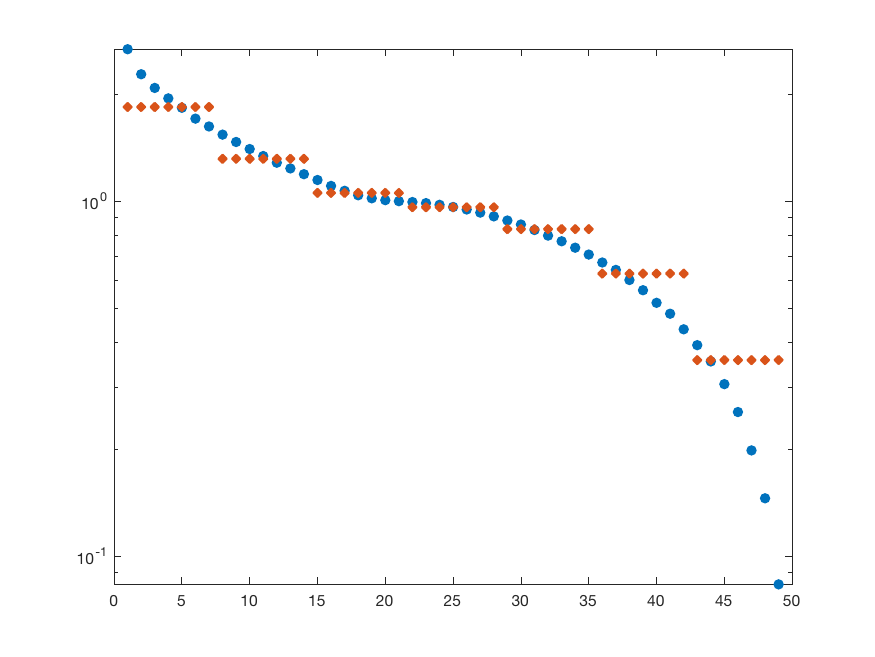}};
  \node[left=of img, node distance=0cm, rotate=90, anchor=center,yshift=-0.7cm] {Magnitude};
 \end{tikzpicture}
\end{minipage}

\begin{minipage}{0.6\textwidth}
\begin{tikzpicture}
  \node (img)  {\includegraphics[scale=0.57]{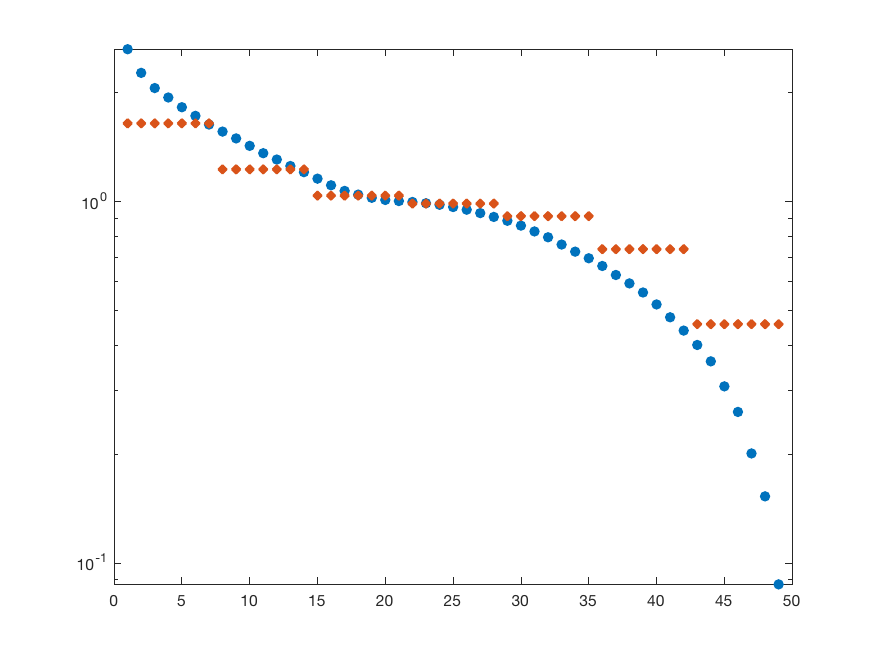}};
  \node[below=of img, node distance=0cm, yshift=1cm] {Index of eigenvalue};
  \node[left=of img, node distance=0cm, rotate=90, anchor=center,yshift=-0.7cm] {Magnitude};
 \end{tikzpicture}
\end{minipage}
\caption{Middle figure: the averaged eigenvalues of 100 trials of forming $A=\kappa_6(X,X)$ (blue dots) and 100000 trials of forming $B=\kappa_6(Y,Y)$ (red dots), where $Y$ is formed from $X$ as before and $\kappa_6:\mathbb{R}^3\times\mathbb{R}^3\to\mathbb{R}$ is defined by $\kappa_6(\mathbf{x},\mathbf{y})=e^{-30|\mathbf{x}-\mathbf{y}|^2}$. Top and bottom figures: we do the same except we (incorrectly) set $D=2$ and $D=4$, respectively, when forming $Y$. Since these are the wrong values of $D$, we get worse quantile estimates in the top and bottom figures: in the top and bottom estimates, 6 and 4, respectively, out of 47 eigenvalues are incorrectly estimated.}\label{fig:baddvalue}
\end{figure}

Therefore, if we start with the collection of points $X$ and restrict it to a small cube $V$ in $\mathbb{R}^D$ where the points are approximately uniformly distributed we can use our eigenvalue quantile estimation technique to see if we get accurate bounds after setting $D$ to several candidate values. That is, we could sample e.g. $n=49$ and $k=7$ points from $V$ and observe which value of $D$ works best to give quantile estimates. In doing so, we would be assuming that our points are ``locally uniformly" distributed (i.e. almost uniform on an appropriate, small-enough ``chart" contained in $V$), and that the embedding generating $X$ restricted to $V$ guarantees that $\kappa(\mathbf{x},\mathbf{y})$ is far from 0 only for points $\mathbf{x}$ and $\mathbf{y}$ that are close within the latent space. In making these assumptions, this setup effectively provides another local notion of intrinsic dimension. Because of its locality, this notion of dimension may be related to the existing $k$-nearest-neighbor-type estimators.

\section*{Acknowledgements}The author sincerely thanks Levon Nurbekyan for his support and ideas related to the theoretical aspects of this work, as well as Yuanzhe Xi and Shifan Zhao for the overall direction of the work and several key observations motivating it.

\end{document}